\documentclass{article}

\PassOptionsToPackage{sort}{natbib}

\usepackage[preprint]{neurips_2019}

\usepackage[utf8]{inputenc} 
\usepackage[T1]{fontenc}    
\usepackage{hyperref}       
\usepackage{url}            
\usepackage{booktabs}       
\usepackage{amsfonts}       
\usepackage{nicefrac}       
\usepackage{microtype}      

\usepackage{dsfont}
\usepackage{amssymb}
\usepackage{amsmath}
\usepackage{amsthm}
\usepackage{color}
\usepackage{algorithmic}
\usepackage{algorithm}
\usepackage{graphicx}
\usepackage{multirow}
\usepackage{enumitem}
\usepackage{verbatim}
\usepackage{hyperref}
\usepackage{microtype}
\usepackage{geometry}
\usepackage{footnote}
\usepackage{chngcntr}
\usepackage{caption}
\usepackage{subfig}
\usepackage{natbib}
\usepackage{xspace}

\newtheorem{theorem}{Theorem}[section]
\newtheorem{Defn}[theorem]{Definition}
\newtheorem{lemma}[theorem]{Lemma}
\newtheorem{claim}[theorem]{Claim}

\newcommand{\calX}{\mathcal{X}}

\newcommand{\calY}{\mathcal{Y}}

\newcommand{\grad}{\nabla}

\def\calX{\mathcal{X}}

\def\calF{\mathcal{F}}

\def\sgn{\text{sgn}}

\newcommand{\thetastar}{\theta^*}
\newcommand{\thetatilde}{\widetilde{\theta}}
\newcommand{\thetatildezero}{\widetilde{\theta_0}}

\newcommand{\norm}[1]{\left\lVert#1\right\rVert}

\def\calG{\mathcal{G}}
\def\vecu{\mathbf{u}}

\newcommand{\SimplisticAttack}{\textsc{SimplisticAttack}}
\newcommand{\easy}{easy}
\newcommand{\hard}{hard}
\newcommand{\intermediate}{intermediate}
\newcommand{\mypara}[1]{{\bf #1}}

\newcommand{\Straight}{Simplistic\xspace}

\title{An Investigation of Data Poisoning Defenses\\
for Online Learning}

\author{
  Yizhen Wang \\
  University of California, San Diego \\
  \texttt{yiw248@eng.ucsd.edu}
  \And
  Somesh Jha \\
  University of Wisconsin, Madison \\
  \texttt{jha@cs.wisc.edu}
  \And
  Kamalika Chaudhuri\\
  University of California, San Diego \\
  \texttt{kamalika@cs.ucsd.edu} 
}

\begin{document}
\maketitle

\begin{abstract}
Data poisoning attacks -- where an adversary can modify a small fraction of training data, with the goal of forcing the trained classifier to high loss -- are an important threat for machine learning in many applications. While a body of prior work has developed attacks and defenses, there is not much general understanding on when various attacks and defenses are effective. 

In this work, we undertake a rigorous study of defenses against data poisoning for online learning. First, we study four standard defenses in a powerful threat model, and provide conditions under which they can allow or resist rapid poisoning. We then consider a weaker and more realistic threat model, and show that the success of the adversary in the presence of data poisoning defenses there depends on the ``ease'' of the learning problem. 
\end{abstract}

\section{Introduction}
Data poisoning attacks are an important threat for machine learning in applications where adversaries are incentivized to influence the behavior of a classifier -- for example, spam filtering~\citep{nelson2008exploiting}, malware detection~\citep{newsome2006paragraph}, sentiment analysis \citep{newell2014practicality} and collaborative filtering \citep{li2016data}. A data poisoning adversary can add or modify a small fraction of training examples -- typically with the goal of forcing the classifier trained on the resulting data to low accuracy. For example, a spammer may introduce carefully crafted spam in order to ensure that a classifier trained on the collected data is inaccurate.  

A body of prior work has developed many poisoning attacks and defenses~\citep{biggio2012poisoning, meizhu, munoz2017towards, shafahi2018, koh2018stronger, steinhardt2017certified, li2016data, paudice2018detection}. However, the literature largely suffers from two main limitations. First, most work assumes that all data is provided in advance and the learner's goal is to produce a loss minimizer on this data. This excludes many applications such as malware detection where classifiers are incrementally updated with new data as well as algorithms such as stochastic gradient descent. The second drawback is that there is little analysis, and not much is rigorously understood about exactly how powerful the attacks are, and the conditions under which existing defenses are successful. 

In this work, we consider data poisoning for a different setting -- online learning, and carry out a theoretical and empirical study of its effectiveness under four popular defenses. In online learning, examples arrive sequentially, and at iteration $t$, the learner updates its parameters based on the newly-arrived example. We select as our learner perhaps the most basic yet widely-used algorithm -- online gradient descent (OGD). Starting with an initial parameter value and a loss function, OGD updates its parameter value by taking a small step against the gradient of the loss at the current example. 

We consider two threat models. First, for its amenability to analysis, we consider a powerful semi-online threat model by~\cite{wang2018data}, where the adversary can add examples in any position in the stream, and their goal is to attack the classifier produced at the end of learning. In addition, we consider a messier but more realisitic fully-online threat model -- where the adversary can only add examples at pre-specified positions and their goal is to achieve high overall loss throughout the entire stream. We assume that the learner has a defense in place that filters out input examples with certain properties determined by the defense mechanism.

For the semi-online threat model, we assume that the adversary has a target classifier in mind, and we measure how rapidly it can cause the learnt classifier to move towards the target when a specific defense is in place. We look at three regimes of effectiveness.  In the first \easy\ regime, there is a simplistic attack that the adversary can use to rapidly reach its target. In the second \hard\ regime, there are no successful data poisoning attacks that succeed against a defense. In between lies an \intermediate\ regime where effective attacks exist, but are not as simplistic or powerful. Specifically, our contributions are:

\begin{itemize}[leftmargin=*, topsep=5pt]
\item We prove that the $L_2$-norm defense, which filters out all examples outside an $L_2$-ball is always in the \easy\ regime and hence mostly weak. 
\item We provide conditions under which the labeling oracle defense~\citep{shafahi2018, suciu2018}, where the adversary {\em{only}} provides an unlabeled example to be labeled by an annotator, is in the \easy\ regime or in the \hard\ regime. 
\item We characterize the performance of two data-dependent defenses  -- the $L_2$-distance-to-centroid defense and the Slab defense~\citep{steinhardt2017certified,koh2018stronger} -- when initialized on clean data. 

\item We empirically investigate the extent to which our theoretical observations translate to practice, and show that existing attacks are mostly effective in the regime that our theory predicts as easy, and mostly ineffective in those predicted as hard. 
\end{itemize}

While our theory directly measures the speed of poisoning, a remaining question is how the poisoning attacks and defenses interact with the learning process overall -- as defenses also filter out clean examples.  To understand this, we consider the interplay of poisoning and learning in a more realistic threat model -- the fully online model -- where the adversary can only add examples in pre-specified positions, and its goal is to force high overall loss over the entire stream.  Our findings here are as follows:

\begin{itemize}[leftmargin=*]
\item Theoretically, we show two examples that illustrate that the adversary's success depends on the ``easiness'' of the learning problem -- for low dimensional data with well-separated classes, defenses work well, while the adversary can succeed more easily for higher dimensional, lower margin data.

\item Experimentally, we corroborate that data poisoning defenses are highly effective when the classification problem is easy, but the findings from the semi-online case carry over when the classification problem is more challenging. 

\end{itemize}

Our results have two major implications for data poisoning in the online setting. First, they indicate that the Slab defense may be a highly effective defense overall. Second, our experiments indicate that for challenging classification problems, weaker threat models can still result in fairly powerful attacks, thus implying that data poisoning is a threat even for these weaker adversaries.

\paragraph{Related Work.} The work closest in spirit to ours is by \cite{koh2018stronger}, which proposes and analyzes strong data poisoning attacks that break a number of data sanitization based defenses. Their work however differs from ours in two major ways. First, they are in the offline setting where all data is provided in advance and the learner is an offline empirical risk minimizer; second, their analysis focuses on the number of distinct poisoning examples needed while ours characterizes when defenses are effective against attacks. \cite{mahloujifar2018curse, mahloujifar2017learning} provides upper bounds on the number of poisoning points required to poison online and offline learners; however, they do not consider specific defenses, and their results only hold for certain classes of distributions. \cite{diakonikolas2018sever} provides a convex optimization algorithm that is resistant to data poisoning attacks and has provable guarantees; however, their algorithm is also in the offline setting and requires advance knowledge of the data. 

Data poisoning is a threat to a number of applications~\citep{nelson2008exploiting, newsome2006paragraph, newell2014practicality, li2016data}, and hence there is extensive prior work on it -- mostly in the offline setting. Efficient attacks have been developed against logistic regression, SVM and neural networks~\citep{biggio2012poisoning, meizhu, munoz2017towards, shafahi2018, koh2018stronger}, and defenses have been analyzed by~\cite{steinhardt2017certified} and~\cite{li2016data}.

Recent work has looked at data poisoning in the online setting against classification~\citep{wang2018data, lessard2018optimal, zhu2019}, contextual bandits~\citep{ma2018data} and autoregressive models~\citep{alfeld2016data}. Prior works on online attacks include~\cite{wang2018data}, which finds an efficient gradient-ascent attack method for a number online objectives, and~\cite{lessard2018optimal}, which views the attack process as an optimal control problem to be solved by a nonlinear optimization solver or be approximated using reinforcement learning~\citep{zhu2019}. However, little is understood on the defense side, and our work is the first to provide theoretical performance guarantees of attacks and defenses for online data poisoning. 
Finally, there has also been prior work on outlier detection by ~\cite{diakonikolas2016efficient,diakonikolas2017being,diakonikolas2018robustly}; unlike ours, these are largely in the unsupervised learning setting.

\section{The Setting}
There are two parties in a data poisoning process -- the learner and the attacker.
We describe the learning algorithm as well as the attacker's capability and goals of two threat models -- semi-online and fully-online.
\subsection{Learning Algorithm}
We consider online learning -- examples $(x_t, y_t)$ arrive sequentially, and the learner uses $(x_t, y_t)$ to update its current model $\theta_t$, and releases the final model $\theta_{T}$ at the conclusion of learning. Our learner of choice is the popular online gradient descent (OGD) algorithm~\cite{shalev2012online, hazan2016introduction}. The OGD algorithm, parameterized by a learning rate $\eta$, takes as input an initial model $\theta_0$, and a sequence of examples $S = \{ (x_0, y_0), \ldots, (x_{T-1}, y_{T-1}) \}$; at iteration $t$, it performs the update:
\[ \theta_{t+1} = \theta_t - \eta \grad \ell(\theta_t, x_t, y_t) \]
where $\ell$ is an underlying loss function -- such as square loss or logistic loss. 
In this work, we focus on logistic regression classifiers, which uses logistic loss.

\subsection{Threat Models}
We consider two types of attackers -- {\bf semi-online} and {\bf fully-online} -- which differ in their poisoning power and objectives. 

\mypara{Semi-online.} 
A semi-online attacker adds poisoning examples to the clean data stream $S$,
and the resulting poisoned stream $S'$ becomes the input to the online learner.
We consider a strong attacker that knows the entire clean data stream and can add up to $K$ examples to $S$ at \emph{any} position.  
Let $\theta$ be the final model output by the learner. 
A semi-online attacker's goal is to obtain a $\theta$ that satisfies certain objectives, for example $\theta$ could be equal to a specific $\theta^*$, or have high error on the test data distribution. 
The attacker is `semi-online' because the learner is online but 1) the attacker's knowledge of the stream resembles an offline learner's, and 2) the objective only involves the final model $\theta$.

\mypara{Fully-online.} 
Unlike a semi-online attacker which knows the entire data stream and can poison at any position, a fully-online attacker only knows the data stream up to the current time step and can only add poisoning examples at pre-specified positions.
This is a more realistic setting in which the learner gets data from poisoned sources at specific time steps. 
Let $S'$ be the resulting poisoned stream of length $T'$ and $I = \{t_1, \cdots, t_K\}$ be the set of time steps with poisoned examples. The fully-online attacker aims to increase the online model's loss over \emph{clean} examples in the data stream over the entire time horizon, i.e. to maximize
$\sum_{t=0}^{T'-1} \ell(\theta_t,x_t,y_t) \mathds{1}[t \notin I]$.

Note that the attacks are white-box in both settings, which is a common assumption in data poisoning. The only work in a general black-box setting is by \cite{dasgupta2019teaching}, whose methods are computationally inefficient. 
Other work either considers simple and limited attacker model~\citep{zhao2017efficient} or knows the model class and learning algorithm~\citep{liu2017towards}. 
The white-box setting also rules out security by obfuscation.

\mypara{Defense.}
In both semi-online and fully-online settings,
we assume that the attacker works against a learner equipped with a defense mechanism characterized by a feasible set $\calF$. For example, $\calF$ is an $L_2$-ball of radius $R$ in the popular $L_2$-norm defense.  
If the incoming example $(x_t, y_t) \in \calF$, then it is used for updating $\theta_t$; otherwise it is filtered out.
The attacker knows $\calF$ under the white-box assumption. 

\section{Analysis}
\label{sec:analysis}

In this section, we analyze the effectiveness of common defenses against data poisoning in the semi-online setting. We formalize a semi-online attacker as follows. 
\begin{Defn}[Data Poisoner]
A Data Poisoner $DP^{\eta}(\theta_0, S, K, \calF, \theta^*, \epsilon)$, parameterized by a learning rate $\eta$, takes as input an initializer $\theta_0$, a sequence of examples $S$, an integer $K$, a feasible set $\calF \subseteq \calX\times\calY$, a target model $\theta^*$, and a tolerance parameter $\epsilon$ and outputs a sequence of examples $\tilde{S}$. The attack is said to succeed if $\tilde{S}$ has three properties --  first, an $OGD$ that uses the learning rate $\eta$, initializer $\theta_0$ and input stream $\tilde{S}$ will output a model $\theta$ s.t. $\norm{\thetastar-\theta}\leq\epsilon$; second, $\tilde{S}$ is obtained by inserting at most $K$ examples to $S$, and third, the inserted examples lie in the feasible set $\calF$. 
\end{Defn}
For simplicity of presentation, we say that a data poisoner outputs a model $\theta$ if the OGD algorithm obtains a model $\theta$ over the data stream $\tilde{S}$.

In order to test the strength of a defense, we propose a simple data poisoning algorithm -- the \SimplisticAttack. A defense is weak if this attack succeeds for small $K$, and hence is able to achieve the poisoning target rapidly. In contrast, a defense is strong if no attack with any $K$ can achieve the objective. We analyze the conditions under which these two regimes hold for an online learner using logistic loss for binary classification tasks. 

\begin{algorithm}
\caption{\SimplisticAttack($\theta_0$, $S$, $K$, $\calF$, $\thetastar$, $\epsilon$, $R$)}
\begin{algorithmic}
\STATE {\bfseries Input:} initial model $\theta_0$, clean data stream $S$, max number of poisoning points $K$, feasible set $\calF$, target model $\thetastar$, tolerance parameter $\epsilon$, max $L_2$ norm $R$ of inputs.
\STATE{Initialize $\thetatildezero$ as the model learned over $S$.}
\STATE{$\lambda = \norm{\thetastar}, \tilde{S} = S, t=0$}
\STATE{$\gamma_0 = \min\left(R/\|\thetatilde_0 - \thetastar\|,1/\eta\right)$}
\STATE{{\bf while} \mbox{$t < K$ and $\|\thetatilde_t - \thetastar\| \geq \epsilon$} {\bf do}}
\STATE{\quad \mbox{$\gamma_t = \min\left(\gamma^*_t, \gamma_0\right)$, where $\gamma^*_t$ is the solution of $\gamma$ to} \mbox{ $\frac{\gamma}{1+\exp(\thetatilde^{\top}_t(\thetastar-\thetatilde_t)\gamma)} = \frac{1}{\eta}$ }}
\STATE{\quad $(x_t, y_t) = (\gamma_t(\thetastar-\thetatilde_t), +1)$}
\STATE{\quad Find the closest $c\in[0,R/\|x_t\|]$ to $1$, s.t. $(cx_t, y_t)\in\calF$ or $(-cx_t, -y_t)\in\calF$.}
\STATE{\quad {\bf if} $(cx_t, y_t)\in\calF$ {\bf then} \mbox{append ($cx_t, y_t$) to $\tilde{S}$}}
\STATE{\quad {\bf else} \mbox{append ($-cx_t, -y_t$) to $\tilde{S}$}}
\STATE{\quad $t=t+1$}
\STATE{{\bf endwhile}}
\STATE{{\bf return} $\tilde{S}$}
\end{algorithmic}
\label{alg:simplisticattack}
\end{algorithm}

The \SimplisticAttack\ algorithm is described in Algorithm~\ref{alg:simplisticattack}. Suppose the online learner has an initial model $\theta_0$ and obtains a model $\thetatildezero$ after learning over the clean stream $S$. 
\SimplisticAttack\ iteratively appends poisoning points to the clean stream $S$, where all poisoning points are in the direction of $\thetastar-\thetatildezero$ with adaptive $L_2$ magnitude capped at $R$. 
The algorithm is inspired by the iterative teaching method~\citep{liu2017iterative} and is modified with an additional projection step for general feasible set $\calF$.  
In order to maintain the direction of the poisoning points, the projection step scales the input by a nonnegative factor $c$.

\subsection{Bounded $L_2$ Norm Defense.}
In many applications, the learner requires the input vector to be within a bounded domain, for example having bounded $L_p$ norm. We consider inputs with bounded $L_2$ norm, which corresponds to a defense mechanism with a feasible set $\calF = \{(x,y)|\norm{x}_2 \leq R\}$ for some constant $R$. Notice that if the input domain is bounded by $L_\infty$ or $L_1$ norms, we can still find an $L_2$ ball inscribed in the domain. We next show that \SimplisticAttack\ is a very efficient data poisoner for this setting. 

\begin{theorem} 
\label{thm:boundedpoisoner}
Let $\theta_0$ and $\thetastar$ be the initial and the target model.
Suppose the feasible set $\calF = \{(x,y)|\norm{x}_2\leq R\}$ and $\norm{\thetastar} = \lambda$.
If $K > C \log (\lambda/\epsilon)$ for some constant $C=C(R)$, then \SimplisticAttack$\left(\theta_0, S, K, \calF, \thetastar, \epsilon, R\right)$ outputs a $\theta$ such that $\norm{\theta - \thetastar} \leq \epsilon$. 
\end{theorem}
\mypara{Remark.} Notice that $K$ is an upperbound of number of poisoning examples needed to output a $\theta$ close to $\thetastar$, therefore also corresponds to a lowerbound on the speed of poisoning. Theorem~\ref{thm:boundedpoisoner} suggests that \SimplisticAttack\ outputs a $\theta$ close to $\thetastar$ within logarithmic number of steps w.r.t. $1/\epsilon$. Therefore, the defense is in the \easy\ regime. The factor $C$ increases with decreasing $R$, and is in the order of $O(1/\eta R)$ when $R$ is small. However, in real applications, the learner cannot set $R$ to be arbitrarily small because this also rejects clean points in the stream and slows down learning. We investigate this in more details in our evaluation in Sec~\ref{sec:exp}.  

\subsection{The Labeling Oracle Defense}
In many applications, the adversary can add unlabeled examples to a dataset, which are subsequently labeled by a human or an automated method. We call the induced labeling function a labeling oracle. Denoting the labeling oracle by $g$, we observe that this oracle, along with bounded $L_2$ norm constraint on the examples, induces the following feasible set:
\begin{equation} \label{eqn:feasible}
\calF = \{ (x, y) | \|x\| \leq R, y = g(x) \} ,
\end{equation}
For simplicity, we analyse the oracle defense on a different feasible set $\calF'=\{(yx, +1)|(x,y)\in \calF\}$ derived from $\calF$. We call $\calF'$ the \emph{one-sided} form of $\calF$ because it flips all $(x,-1)\in\calF$ into $(-x,+1)$ and as a result only contains points with label $+1$. Lemma~\ref{lem:equivalentf} in the appendix shows that an attacker that outputs $\theta$ using $K$ points in $\calF$ always corresponds to some attacker that outputs $\theta$ using $K$ points in $\calF'$ and vice versa. Therefore, the defenses characterized by $\calF$ and $\calF'$ have the same behavior.

We next show that for feasible sets $\calF'$ of this nature, three things can happen as illustrated in Figure~\ref{fig:general}. First, if $\calF'$ contains a line segment connecting the origin $O$ and some point in the direction of $\theta^* - \thetatildezero$, then \SimplisticAttack\ can output a $\theta$ close to $\theta^*$ rapidly. Second, if $\calF'$ is within a convex region $\calG$ that does not contain any point in the direction of $\theta^*-\theta_0$, then no poisoner can output $\thetastar$. Third, if $\calF'$ contains points in the direction of $\theta^* - \thetatildezero$ but not the origin, then the attack can vary from impossible to rapid. Theorem~\ref{thm:generalconstraints} captures the first and the second scenarios, and Appendix~\ref{sec:intermediatecases} shows various cases under the third.

\begin{figure}
\centering
\includegraphics[width=0.5\textwidth]{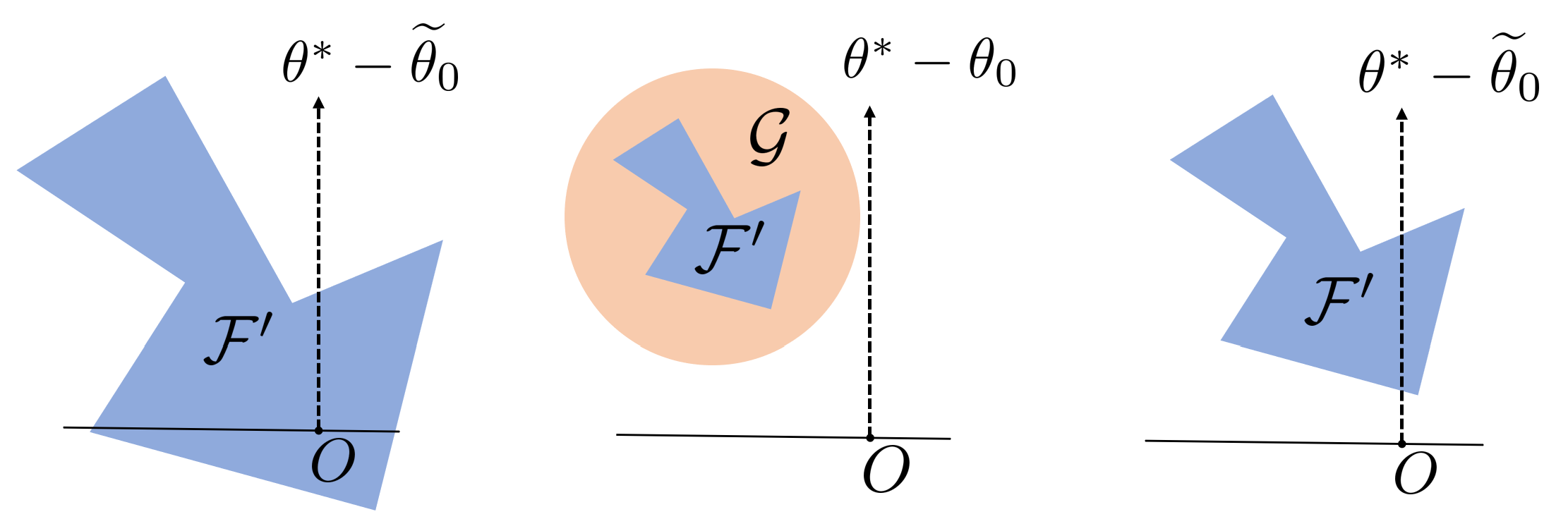}
\caption{
Schematic illustration of the one-sided feasible sets $\calF'$ corresponding to three defense regimes.
Left to right: defense is in the \easy, \hard\ and \intermediate\ regime.
} 
\label{fig:general}
\end{figure}

\begin{theorem}
Let $L(r, \vecu) = \left\{c\vecu/\norm{\vecu}\Big|0 < c \leq r \right\}$ denote a line segment connecting the origin and $r\vecu/\|\vecu\|$ for some vector $\vecu$. Also, let $\theta_0, \thetastar$ be as defined in Theorem~\ref{thm:boundedpoisoner} and $\thetatildezero$ be the online learner's model over the clean stream $S$. Suppose $\norm{\thetastar} = \lambda$.
\begin{enumerate}
\item
If $L(r, \theta^* - \thetatildezero) \subseteq \calF'$ for some $r > 0$ and $K \geq C \log (\lambda/\epsilon)$ for some constant $C$, then \SimplisticAttack$\left(\theta_0, S, K, \calF', \thetastar, \epsilon, r\right)$ outputs a $\theta$ with $\norm{\theta - \theta^*} \leq \epsilon$.
\item
If there exists a convex set $\calG$ such that 1) $\calF'\subseteq\calG$, and 2) $\calG \bigcap L(+\infty,\theta^*-\theta_0) = \emptyset$, then no data poisoner can output $\thetastar$ for any $K$.
\end{enumerate}
\label{thm:generalconstraints}
\end{theorem}

\subsection{Data-driven Defenses}
A common data poisoning defense strategy is to find a subset of the training points that are anomalous or ``outliers", and then sanitize the data set by filtering them out before training. 
Following~\cite{cretu2008},~\cite{steinhardt2017certified},~\cite{paudice2018} and~\cite{koh2018stronger}, we assume a defender who builds the filtering rule using a clean initialization data set that the poisoner cannot corrupt. Thus the poisoner knows the defense mechanism and parameters, but has no way of altering the defense. We focus on two defenses, the $L_2$-distance-to-centroid defense and the Slab defense~\citep{steinhardt2017certified,koh2018stronger}, which differ in their outlier detection rules.

\mypara{Defense Description.}
Both defenses assign a score to an input example, and discard it if its score is above some threshold $\tau$. The $L_2$ norm constraint on the input vector still applies, i.e. $\norm{x} \leq R$. 
Let $\mu_+$ denote the centroid of the positive class and $\mu_-$ the centroid of the negative class computed from the initialization set.

$L_2$-distance-to-centroid defense assigns to a point $(x,y)$ a score $\norm{x-\mu_y}$. This induces a feasible set $\calF^c_+ = \{(x,+1)|\norm{x-\mu_+}\leq \tau, \norm{x}\leq R\}$ for points with label $+1$ and $\calF^c_- = \{(x,-1)|\norm{x-\mu_-}\leq \tau, \norm{x}\leq R\}$ for points with label $-1$. The entire feasible set $\calF^c = \calF^c_+ \bigcup \calF^c_-$.

Slab defense assigns to a point $(x,y)$ a score $|(\mu_+-\mu_-)^{\top}(x-\mu_y)|$. We call $\beta = \mu_+ - \mu_-$ the ``defense direction". Intuitively, the defense limits the distance between the input and the class centroid in the defense direction. The defense induces a feasible set $\calF^s_+ = \{(x,+1)||\beta^{\top}(x-\mu_+)| \leq \tau, \norm{x}\leq R\}$ for points with label $+1$ and $\calF^s_- = \{(x,-1)||\beta^{\top}(x-\mu_-)| \leq \tau, \norm{x}\leq R\}$ for points with label $-1$. The entire feasible set $\calF^s = \calF^s_+ \bigcup \calF^s_-$.

\mypara{$L_2$-distance-to-centroid Defense.}
If $\tau$ is large such that either $\calF^c_+$ or $\calF^c_-$ contains the origin, then \SimplisticAttack\ can succeed rapidly. 
If $\tau$ is small such that the one-sided form of $\calF^c$ does not contain any point on the positive side of the hyperplane with normal vector $(\theta^*-\theta_0)$, then no attack is possible. 
Lemma~\ref{lem:l2centroid} shows the condition for these two cases.  

\begin{lemma}
Let $u_+$ be $\mu_+$'s projection on $\thetastar-\theta_0$. Similarly, let $u_-$ be $\mu_-$'s projection on $\theta_0-\thetastar$.
\begin{enumerate}

\item 
If $\tau > \min(\norm{\mu_+}, \norm{\mu_-})$ and $K \geq C \log (\lambda/\epsilon)$ for some constant $C$, then $\exists r>0$ such that \SimplisticAttack$\left(\theta_0, S, K, \calF^c, \thetastar, \epsilon, r\right)$ outputs a $\theta$ with $\norm{\theta- \theta^*} \leq \epsilon$. 
\item Otherwise, if $\langle\mu_+, \thetastar-\theta_0\rangle < 0$, $\langle\mu_-, \theta_0-\thetastar\rangle < 0$ and $\tau \leq \min(\norm{u_+}, \norm{u_-})$, then no data poisoner can output $\thetastar$ for any $K$.  
\end{enumerate}
\label{lem:l2centroid}
\end{lemma}

\mypara{Slab Defense.} 
Each of $\calF^s_+$ and $\calF^s_-$ is a `disc' between two hyperplanes with normal vector $\beta$. 
When $\tau$ is large such that either $\calF^s_+$ or $\calF^s_-$ contains the origin, \SimplisticAttack\ can succeed rapidly.
When $\tau$ is small and the projection of $y(\theta^*-\theta_0)$ on $\beta$ is opposite to $\mu_y$, then no attack is possible.
Lemma~\ref{lem:slab} shows the condition for both cases.

\begin{lemma}
Let $\beta = (\mu_+-\mu_-), b_+ = -\beta^{\top}\mu_+, b_-=\beta^{\top}\mu_-$. 
WLOG, assume $\beta^{\top}(\thetastar-\thetatildezero) \geq 0$.\footnote{If $\beta^{\top}(\thetastar-\thetatildezero) < 0$, we can set $\beta=(\mu_--\mu_+)$. The Slab score remains the same.}

\begin{enumerate}
\item
If $\tau-b_+>0>-\tau-b_+$ or $\tau-b_->0>-\tau-b_-$, then there exist a data poisoner $DP^{\eta}(\theta_0, S, K, \calF^s, \thetastar, \epsilon)$
that outputs a $\theta$ with $\norm{\theta - \theta^*} \leq \epsilon$ for $K \geq C \log (\lambda/\epsilon)$ for some constant $C$. 
\item
If $0 \geq \tau-b_+ > -\tau-b_+$ and $0 \geq \tau-b_- > -\tau-b_-$, then \SimplisticAttack\ cannot output $\thetastar$ for any $K$. In addition, if $\beta^{\top}(\theta^*-\theta_0) \geq 0$, then no data poisoner outputs $\thetastar$ for any $K$.
\end{enumerate}
\label{lem:slab}
\end{lemma}

\section{The Fully-Online Setting}
\label{sec:fully-online}
Our theory in Section~\ref{sec:analysis} discusses the effectiveness of attack and defense in the semi-online setting.
Do effective semi-online attacks always lead to effective fully-online attacks, where the attacker wants the online model to have high loss over the entire time horizon? 
We show by two examples that the result varies on different learning tasks, and thus the existence of effective semi-online attack alone is not sufficient.
The poisoning effect is quickly negated by clean points in one case but not in the other, which suggests the effectiveness of fully-online attack also critically depends on the difficulty of the learning tasks. As a result, the learner also needs to consider its defense mechanism's impact on the learning process. 

\mypara{Effectiveness of Poisoning on Different Tasks.}
We consider the following learner-attacker pair. The learner uses the standard OGD algorithm with initial value $\theta_0 = 0$, learning rate $\eta=1$, and also an $L_2$ norm defense with $\calF = \{(x,y)|\norm{x}\leq 1\}$, which is always in easy regime for semi-online attacks as shown in Theorem~\ref{thm:boundedpoisoner}.
The attacker can inject $10\%$ as many poisoning points as the clean points at any position.
We then show two different tasks in Lemma~\ref{lem:fullyonlinehard} and~\ref{lem:fullyonlineeasy}, in which fully-online attack is hard and easy, respectively.

\begin{lemma}[hard case]
\label{lem:fullyonlinehard}
Consider a data distribution with input $x\in\{1, -1\}$ and label $y = \sgn(x)$. The attacker cannot cause classification errors on more than $10\%$ of the clean examples.\footnote{The bound is tight as an attacker can always cause as many errors as the number of poisoning points under $L_2$ norm defense.}
\end{lemma}

\begin{lemma}[easy case]
\label{lem:fullyonlineeasy}
Consider an input space $\mathbb{R}^d$ with $d=10,000$. The clean inputs are uniformly distributed over $2d$ points $\{e_1, \cdots, e_d, -e_1, \cdots, -e_d\}$, where $e_i$ is the $i$-th basis vector. The label $y=+1$ if $x \in \{e_1, \cdots, e_d\}$, and $y=-1$ otherwise.
Then the attacker can, with probability more than $0.99$, cause classification error on $\geq 50\%$ of the first $1,000,000$ clean examples.
\end{lemma}

In the first case, learning is easy as each clean example contains all the information needed, while in the second, learning is hard as each clean example only contributes to one dimension among many. A fully-online attack is effective if poisoning is easier than learning, and is less likely to succeed otherwise. We further validate this finding over different real-world data sets with different complexity for learning in Section~\ref{sec:exp}.

\mypara{Impact of Defenses.}
In reality, the defense mechanism filters both the clean and the poisoning points.
Therefore, the learner cannot set the feasible set arbitrarily small since such defense also prohibits learning useful patterns.
For example, if the $L_2$ norm bound is less than $1$ for the learner in Lemma~\ref{lem:fullyonlinehard} and~\ref{lem:fullyonlineeasy},
then all clean points will be filtered out, and the model will be solely determined by poisoning points.
The learner needs to find an appropriate defense parameter so the feasible set keeps enough clean points while slows the attack.   
We evaluate the effectiveness of defenses over a range of defense parameters in Section~\ref{sec:exp}.

\section{Experiments}
\label{sec:exp}
\begin{figure*}
\centering
\includegraphics[width=0.9\textwidth]{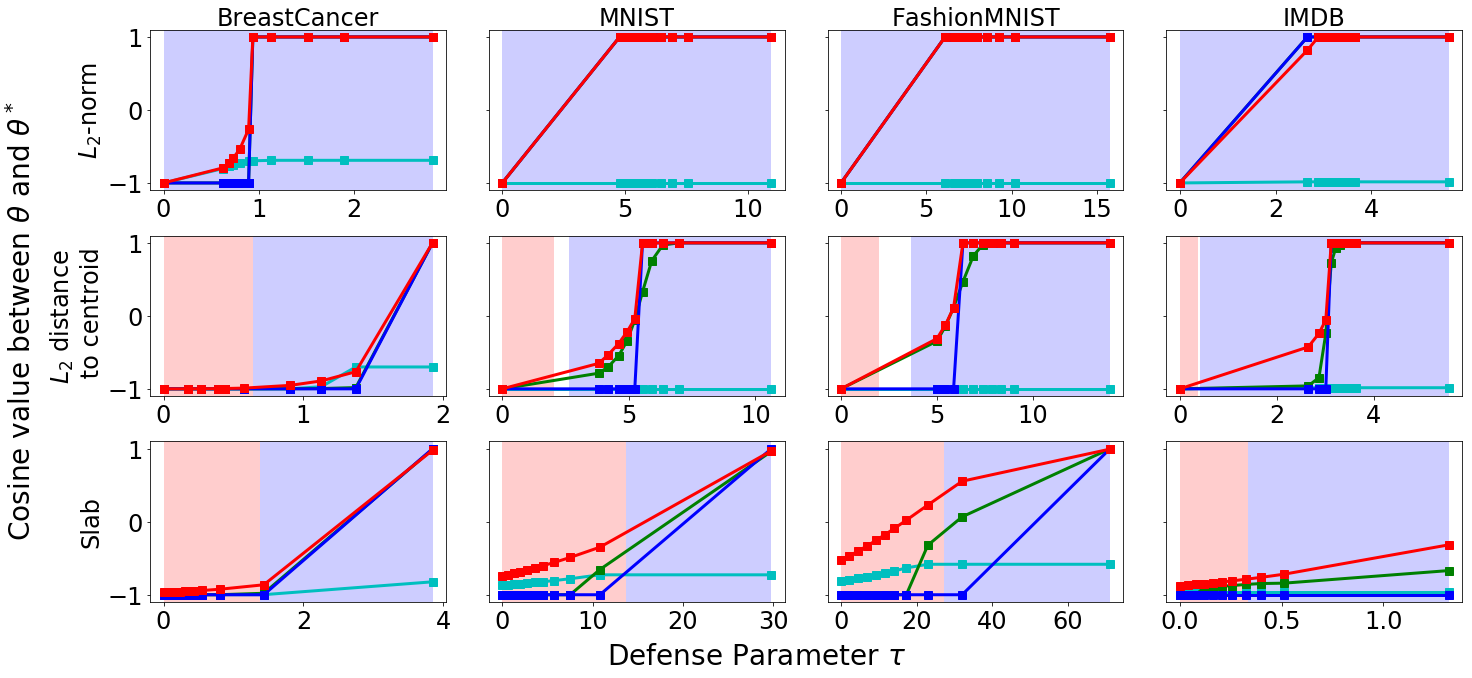}
\includegraphics[height=2.5in]{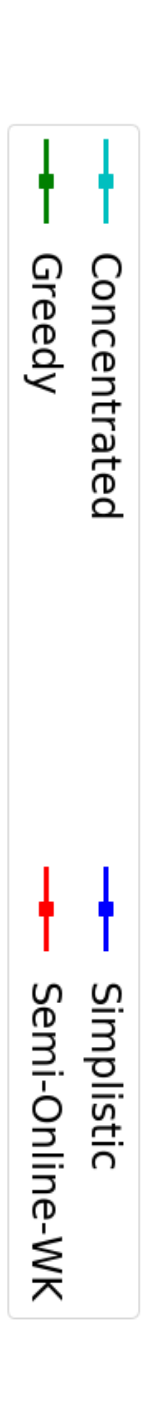}
\caption{The plot of $\cos(\theta, \theta^*)$ against defense parameter $\tau$ for the semi-online attacks where the attacker uses a fixed number of poisoning points.
Small $\tau$ means stronger defense.
Larger $\cos(\theta, \thetastar)$ means more successful attack.
Red background region indicates hard regime, while blue background indicates easy regime.
{\bf Top to bottom}: $L_2$-norm, $L_2$-distance-to-centroid and Slab defense.
{\bf Left to Right}: BreastCancer, MNIST, FashionMNIST, IMDB.
}
\label{fig:semionline}
\end{figure*}

Is our analysis on the rapidity of attacks confirmed by practice? And how does the online classification error rate vary across different defenses and real-world data sets, given that the relative speed of poisoning and learning matters? We now investigate these questions through an experimental study. In particular, we ask the following questions.

\begin{itemize}[leftmargin=*]
\item Do practical data-poisoning defenses exhibit the different regimes of effectiveness that we see in theory in the semi-online setting? 
\item How successful are fully-online attacks across different data sets, given that the tasks have different difficulty levels of learning?
\end{itemize}

These questions are considered in the context of three defenses -- $L_2$-norm, $L_2$-distance-to-centroid and Slab.\footnote{The labeling oracle defense is not included as we lack the appropriate labeling functions for the real-world tasks involved.} The defenses do display different regimes of effectiveness as predicted, and the loss caused by fully-online attacks is more significant on data sets that are harder to learn. 

\subsection{Semi-online Experiment Methodology}
\mypara{Baseline Attacks.} To evaluate the effectiveness of the defenses, we consider four canonical baseline attacks. All attacks we use append the poisoning examples to the end of the clean data stream, as~\cite{wang2018data} shows that this works best for our setting. 

The {\bf{\Straight}} attack uses the \SimplisticAttack\ algorithm.
The {\bf{Greedy}} attack~\citep{liu2017iterative}, finds, at step $t$, the example $(x_t, y_t) \in \calF$ that minimizes $\norm{\theta_{t+1}-\thetastar}$, where $\theta_{t+1}$ is derived from the OGD update rule. 
The {\bf{Semi-Online-WK}} attack~\citep{wang2018data} finds $K$ poisoning examples together by maximizing the loss of the resulting model on a clean validation dataset. The optimization problem is solved using gradient descent. The poisoning points are again inserted at the end.
As a sanity-check, we also include an offline baseline -- the {\bf{Concentrated}} attack~\citep{koh2018stronger}, which is effective against many offline defenses.
We explain its adaption to the online setting in Appendix~\ref{sec:detailedprocedures}.

\mypara{Data Sets.} We consider four real-world data sets -- UCI Breast Cancer (dimension $d=9$), IMDB Reviews~\citep{imdb} ($d=100$),\footnote{The features are extracted using Doc2Vec on the unlabeled reviews by setting $d=100$.} MNIST 1v7 ($d=784$) and FashionMNIST~\citep{fmnist} Bag v.s. Sandal $(d=784)$. The standard OGD algorithm is able to learn models with high accuracy on all clean data sets. 
Each data set is split into three parts: initialization (used for the data-driven defenses), training and test sets. Procedural details are in Appendix~\ref{sec:datapreparation}.

\mypara{Experimental Procedure.} 
The learner starts with $\theta_0 = 0$ and uses the OGD algorithm to obtain a model $\theta$ by iterating over the sequence of examples output by the poisoner.
The attacker sets the target model $\thetastar = -\tilde{\theta}_0$, where $\tilde{\theta}_0$ is the learner's final model if it updates only on the clean stream. 
Poisoning examples are restricted inside the feasible set induced by the learner's defense, and scores for the data-driven defenses are calculated based on the initialization set.
The number of poisoning examples $K$ is $80$ for Breast Cancer, $100$ for MNIST/FashionMNIST and $200$ for IMDB.

To avoid confusion by scaling factors, we evaluate the effect of poisoning by $\cos(\theta, \thetastar)$, the cosine similarity between the final model $\theta$ and the target model $\thetastar$. Successful attacks will result in large $\cos(\theta, \thetastar)$ values.

For $L_2$-norm defense, the defense parameter $\tau$ is the maximum $L_2$ norm of the input; for $L_2$-distance-to-centroid, $\tau$ is the maximum $L_2$ distance between an input to its class centroid, and for Slab, it is the maximum Slab score. 
Smaller $\tau$ means stronger defense.
We consider ten $\tau$ values.
The $i$-th value is the $10i$-th percentile line of the corresponding measure for points in the clean stream. 
We choose $\tau$ data-dependently to ensure the range of $\tau$ is practical. 

\subsection{Fully-online Experiment Methodology}
\mypara{Baseline attacks and data sets.}
Existing fully-online attack methods~\citep{wang2018data, zhu2019} are either too computationally expensive for high dimensional data and long time horizons, or require knowledge of the entire clean stream. 
Instead, we use a generic fully-online attack scheme which reduces to semi-online attacks as follows.
At time $t$ when the learner gets examples from the attacker, the attacker runs a semi-online attack algorithm with $\theta_t$ as the initial value and a model with high fully-online loss as the target $\thetastar$ to generate the poisoning examples.\footnote{Such attack does \emph{not} require knowledge of future clean points, which either needs to be given or estimated in previous methods.} 
We construct three fully-online attackers using {\bf \Straight}, {\bf Greedy} and {\bf Semi-Online-WK} as the semi-online subroutines respectively, and name the baselines after their semi-online subroutines.
We use the same four data sets as in the semi-online experiment. 

\mypara{Experiment Procedure.} 
The sequence length $T$ is $400$ for Breast-Cancer and $1000$ for MNIST, FashionMNIST and IMDB. For each run of experiment, a set of attack position $I$ with $|I| = 0.1T$ is first drawn uniformly random over all possible positions.\footnote{We also try a more powerful poisoner with $|I|=0.2T$; the result is shown in Appendix~\ref{sec:additionalplots} and has similar trend.} 
The attacker generates a poisoning example when $t\in I$, otherwise a clean sample is drawn from the training set.
The learner starts with $\theta_0 = 0$ and updates on the poisoned data stream.
Examples that fall outside the feasible set are not used for updating model.  
At each iteration $t$, the learner also predicts the class label of the input $x_t$ as $f_t(x_t) = \sgn(\theta_t^{\top}x_t)$. 
The effect of poisoning is measured by the online classification error rate over the entire time horizon 
$\sum_{i=0}^{T-1}\mathds{1}(f_t(x_t) \neq y_t)\mathds{1}(t\notin I)$.

For each defense, we use five $\tau$ values, corresponding to five feasible sets $\calF$ which contains 30\%, 50\%, 70\%, 90\% and 100\% of the clean examples in the stream. 

\begin{figure*}
\centering
\includegraphics[width=0.9\textwidth]{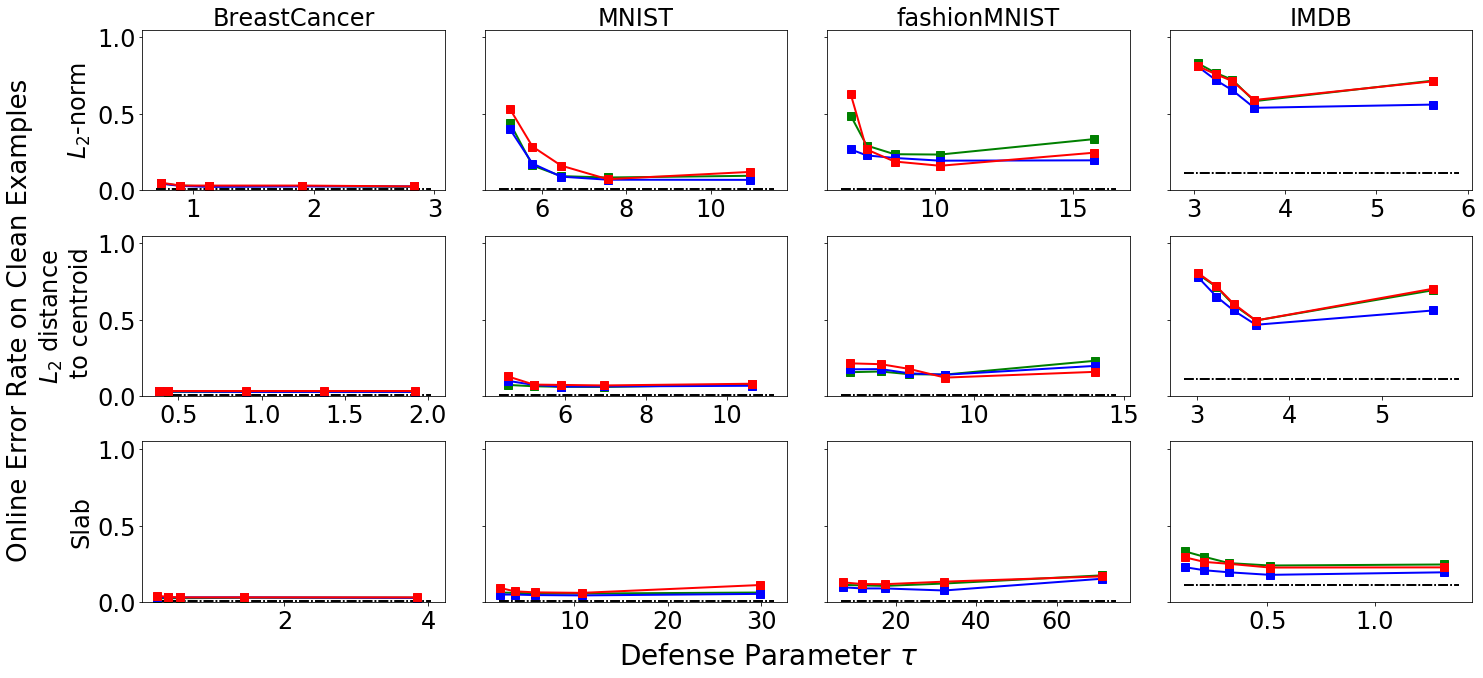}
\includegraphics[height=2.5in]{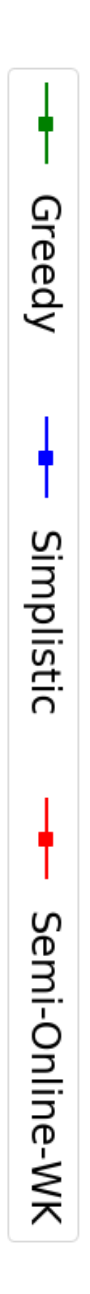}
\caption{Online classification error rate against defense parameter $\tau$ when 10\% of the examples in the data stream comes from the attackers. 
The dashdot line shows the error rate of the offline optimal classifier.
Small $\tau$ means the defense filters out more examples, both clean and poisoned.
Larger online classification error means more successful attack.
{\bf Top to bottom}: $L_2$-norm, $L_2$-distance-to-centroid and Slab defense.
{\bf Left to Right}: BreastCancer, MNIST, FashionMNIST, IMDB.
}
\label{fig:fullyonline}
\end{figure*}

\subsection{Results and Discussion}

Figure~\ref{fig:semionline} presents the cosine similarity between the final model $\theta$ and the target model $\theta^*$ against various defense parameters $\tau$. An attack is rapid if it can make $\cos(\theta, \theta^*)$ close to $1$ within the fixed budget of poisoning examples.
In contrast, a defense is successful if $\cos(\theta, \theta^*)$ is close to $-1$, i.e. $\theta$ is close to the model under no poisoning. 
We also show $\theta$'s error rate on clean test examples in Appendix~\ref{sec:additionalplots}.
Figure~\ref{fig:fullyonline} presents the online classification error rate over clean examples in the stream against defense parameter $\tau$ when $10\%$ of the stream is poisoned.
A defense is effective if the error rate is low for all attack baselines.

\subsubsection{Semi-online Experiments}
\mypara{Overall Performance of Attacks.} \Straight, Greedy and Semi-Online-WK are effective online attacks -- $\cos(\theta, \thetastar)$ increases as $\tau$ increases and approaches $1$ for large $\tau$ in most cases. Concentrated is ineffective for ignoring the online nature of the problem. Semi-Online-WK typically performs the best for its flexibility in poisoning examples selection. 

\mypara{Regime of Effectiveness for Defenses vs. Theories.} In Figure~\ref{fig:semionline}, we highlight the easy and hard regime predicted by our analysis in blue and red background. The boundaries are calculated based on the initialization data set. The results corroborate our theories in the following aspects. 
First, poisoning is more rapid as $\tau$ increases, as larger $\tau$ corresponds to larger feasible set. 
Second, we see from the figure that for most datasets and defences, there are regimes where \SimplisticAttack\ is rapid, and regimes where no attacks work -- e.g. easy and hard. However, their boundaries do not always completely agree with the theoretically predicted boundaries -- as is to be expected. This is because the multiplicative factor $C$ in the lowerbound of poisoning speed implicitly depends on $\tau$; smaller $\tau$ leads to larger $C$, and thus the number of poisoning examples required increases.
We show in an additional experiment in Appendix~\ref{sec:additionalplots} that the attacker can still make $\cos(\theta, \theta^*)$ close to $1$ for small $\tau$ in the predicted easy regime under a larger learning rate $\eta=1$, which lowers $C$.

\mypara{Comparison between Defenses.} $L_2$-norm defense is the weakest as it is always in easy regime. 
Slab, in contrast, is potentially the strongest defense mechanism because it has the widest hard regime -- the defense is in hard regime even when the feasible set keeps $70\%$ to $80\%$ of clean examples.

\subsubsection{Fully-online Experiments}
\mypara{Overall Effectiveness of Defenses.}
On all datasets, Slab defense is the most effective because it can slow down poisoning while still keeping most clean examples.
$L_2$-norm is the least effective because it always permits the attacker to move the online model towards $\theta^*$.
$L_2$-distance-to-centroid is more effective than $L_2$-norm on BreastCancer, MNIST and FashionMNIST but not on IMDB.

\mypara{Online Error Rate across Data Sets and its Implication.}
Over all defense methods, IMDB has the largest online classification error rate and BreastCancer has the lowest.
MNIST and FashionMNIST have similar behaviors.
The results corroborate our intuition on the impact of the learning task's difficulty level. 
BreastCancer is an easy low dimensional learning task, so poisoning is quickly negated by learning on clean examples.
IMDB is a harder task with high dimensional inputs; poisoning is more rapid than learning for weak defenses and the attack can cause more online classifcation error than the number of poisoning examples.
MNIST/FashionMNIST are high dimensional but the inputs are more structured, e.g. all centered and similar in size, therefore the results are in between.

\mypara{Effectiveness of Defenses vs. Choice of Defense Parameter.}
The lowest online classification error rate mostly occurs when $\tau$ is set to include $70\%$ or $90\%$ of the examples in the stream.
Defenses with too small $\tau$ filter out too many clean example and slows down learning, while too large $\tau$ allows fast poisoning. 
The result suggests that a moderate $\tau$ that keeps the majority of clean examples is good in practice.

\section{Conclusion}
In conclusion, we perform a thorough theoretical and experimental study of defenses against data poisoning in online learning.  
In semi-online setting, we show different regimes of effectiveness for typical defenses in theory, and validate the predicted effectiveness in experiment.
In fully-online setting, we show by example that poisoning effect also depends on the difficulty of the learning task, and validate the finding on real-world data sets.
Our experiment suggests that the Slab defense is often highly effective in practice.
For future work, we want to extend the analysis to more complex models such as neural nets, to more attack objectives such as targeted attack at specific test instance, clean-label stealthy attack and backdoor attack, and provide defenses with provable guarantees in each case.

\section{Acknowledgement}
We thank ONR under N00014-16-1-261, UC Lab Fees under LFR 18-548554 and NSF under 1804829for research support. This project is also partially supported by Air Force Grant FA9550-18-1-0166,the National Science Foundation (NSF) Grants CCF-FMitF-1836978, SaTC-Frontiers-1804648 andCCF-1652140 and ARO grant number W911NF-17-1-0405, and we thank for all the supports.

\bibliographystyle{plainnat}
\bibliography{online}

\appendix
\section{Proofs to Theorems and Lemmas}

\subsection{Proof to Theorem~\ref{thm:boundedpoisoner}}
\label{sec:boundedpoisonerproof}
Recall that $\gamma_0=\min\left(\frac{R}{\norm{\thetatildezero-\thetastar}}, 1/\eta\right)$. In round $t$, we pick $\gamma_t = \min(\gamma_t^*, \gamma_0)$,
where $\gamma_t^*$ is the solution to $\frac{\gamma}{1+\exp(\thetatilde^{\top}_t(\thetastar-\thetatilde_t)\gamma)} = \frac{1}{\eta}$. 
Suppose we set $y_t = 1$ and $x_t = \gamma_t (\thetastar - \thetatilde_t)$; this $(x, y)$ lies in $\calF$ from the way we choose $\gamma_t$.
Notice that $\thetatilde^{\top}_t \thetastar - \norm{\thetatilde_t}^2 \leq \lambda^2$. 
If $\gamma_t = \gamma_t^*$, then the poisoner can output $\theta^*$ in the next step. Otherwise if $\gamma_t = \gamma_0$, the norm of $\thetatilde_t - \thetastar$ shrinks geometrically because  

\begin{equation}
\begin{split}
        \quad & \norm{\thetatilde_{t+1}-\theta^*}\\
	= \quad & \norm{\thetatilde_t - \theta^* - \eta \frac{\gamma_0(\theta^* - \thetatilde_t)}{1+\exp\left(\thetatilde^{\top}_t(\theta^* - \thetatilde_t)\gamma_0\right)}}\\
        = \quad & \norm{(\thetatilde_t - \theta^*)\left( 1 - \frac{\eta\gamma_0}{1+\exp\left(\thetatilde^{\top}_t(\theta^* - \thetatilde_t)\gamma_0\right)}\right)}\\
        \leq \quad & \norm{(\thetatilde_t - \thetastar)\left( 1 - \frac{\eta\gamma_0}{1+\exp(\lambda^2\gamma_0)}\right)} \\
        = \quad & \left( 1 - \frac{\eta\gamma_0}{1+\exp(\lambda^2\gamma_0)}\right)\norm{\thetatilde_t - \theta^*}.
\end{split}
\end{equation}
The first inequality holds because $\thetatilde^{\top}_t \thetastar - \norm{\thetatilde_t}^2 \leq \lambda^2$ and the monotonicity of exponential function. The last equality holds because $1 - \frac{\eta\gamma_0}{1+\exp(\lambda^2\gamma_0)} \geq 1 - \frac{\eta\gamma_0}{1+\exp\left(\thetatilde^{\top}_t(\theta^* - \thetatilde_t)\gamma_0\right)} \geq 0$ by our construction.
The result follows from setting $C = \frac{1}{-\log\left(1-\frac{\eta\gamma_0}{1+\exp(\lambda^2\gamma_0)}\right)}$.

Notice that if we let $\lambda=1$, the Taylor expansion of $C$ at $\gamma_0 = 0$ is approximately $2/(\eta\gamma_0)$. If we treat $R$ as a variable, then when $R$ is small, $\gamma_0 = R/\|\thetatildezero-\thetastar\|$, and therefore $C$ is approximately $2\|\thetatildezero-\thetastar\|/(\eta R)$, which is in the order of $O(1/R)$.

\subsection{Statement and Proof to Lemma~\ref{lem:equivalentf}}
\label{sec:equivalentfproof}
\begin{lemma}
\label{lem:equivalentf}
A data poisoner $DP^{\eta}(\theta_0, S, K, \calF, \thetastar, \epsilon)$ can output a model $\theta$ if and only if there exists some data poisoner $DP^{\eta}(\theta_0, S, K, \calF', \thetastar, \epsilon)$ that outputs $\theta$. 
\end{lemma}

We first show that if a data poisoner can output $\theta$ using a sequence $\tilde{S} = \{(x_0, y_0), \cdots, (x_{T-1}, y_{T-1})\}$ with each point $(x_i, y_i) \in \calF$, then there must be a data poisoner that output $\theta$ using a sequence $\tilde{S}' = \{(x'_0, y'_0), \cdots, (x'_{T-1}, y'_{T-1})\}$ with each point $(x'_i, y'_i) \in \calF'$. Notice that for an OGD algorithm running on logistic regression model, $(x,+1)$ and $(-x,-1)$ will lead to the same gradient updates on the model parameter. We construct $\tilde{S}'$ as follows. For each $(x_i, y_i) \in \tilde{S}$, if $(x_i, y_i)\in\calF'$, then we let $(x'_i, y'_i) = (x_i, y_i)$; otherwise, we let $(x'_i, y'_i) = (-x_i, -y'_i)$. Notice by the construction rule of $\calF'$, at least one of $(x_i, y_i)$ and $(-x_i,-y_i)$ exists in $\calF'$. Hence, we have obtained a data poisoner on $\calF'$ using a poisoning sequence of same length and outputing the same model as a poisoner on $\calF$.

The reverse direction is similar. We can use the same technique to show that whenever a poisoner can output $\theta$ using a sequence of length $T$ with points on $\calF'$, there is a sequence using points on $\calF$ that achieves the same goal. Therefore, the statement is true. 

\subsection{Proof to Theorem~\ref{thm:generalconstraints}}
\label{sec:generalconstraintsproof}
For the first part, we use the same poisoning sequence as in the proof of Theorem~\ref{thm:boundedpoisoner}; what remains to be shown is that all $(x_t, y_t)$ in this sequence still lie in the feasible set $\calF$. To show this, we begin with showing that for any $t$, $\thetatilde_t - \theta^* = c (\thetatildezero - \theta^*)$ for some scalar $c \in [0,1]$. 

We prove this by induction. The base case is easy -- $t=0$, and $c = 1$. In the inductive case, observe that:
\[ \thetatilde_{t+1} - \theta^* = \thetatilde_t - \theta^* - \frac{\eta \gamma_t (\thetatilde_t - \theta^*)}{1 + \exp(\thetatilde_t^{\top} (\thetatilde_t - \thetastar)\gamma_t)}, \]
and that $(1 - \eta \gamma_t / (1 + \exp(\thetatilde_t^{\top} (\thetatilde_t - \thetastar)\gamma_t))) \in [0, 1]$ by construction. Thus, the inductive case follows. Now, observe that the proof of Theorem~\ref{thm:boundedpoisoner} still goes through, as $y = 1$ and $x = \gamma (\thetastar-\thetatilde_t) = c \gamma (\thetastar-\thetatildezero)$ is still a feasible point to add to the teaching sequence. 

For the second part, observe that if the data poisoner outputs $\thetastar$, then $\thetastar = \theta_0 - \eta \sum_i \ell'(\theta_{i-1}, x_i, y_i) y_i x_i$ for some $(x_i, y_i)$'s. For the augmented feasible set $\calF'$, $y_i=+1$ always. For logistic regression, $\ell' > 0$ always. Hence, there should exist some collection of $x_i$'s, and some positive scalars $\alpha_i$ such that:
\[ \sum_i \alpha_i x_i = \theta^* - \theta_0. \]
However, Lemma~\ref{lem:impossiblesum} shows that such collection is impossible by letting $\vecu = \thetastar-\thetatilde$ and $\beta=1$. Therefore, no data poisoner can output $\theta^*$.

\begin{lemma}
Let $\calG$ be a convex set, $\vecu$ be a vector and $r\vecu \notin \calG$ for all $r\geq 0$. Then there is no set of points $\{x_1, \cdots, x_n|x_i \in \calG, \forall i\in[n]\}$ such that
\[ \sum_i \alpha_i x_i = \beta \vecu \]
for any $\alpha_i, \beta \in \mathbb{R}^+$ and $n\in\mathbb{Z}^+$.
\label{lem:impossiblesum}
\end{lemma}
\begin{proof}
We prove the statement by induction.
The base case $n=1$ is true because $\calG$ does not contain any point in the direction of $\vecu$. Suppose the statement is true for $n = k$. Assume there exists a set $S=\{x_1, x_2, \cdots, x_{k+1}\}$ such that 
\begin{equation}
\label{eq:1}
\sum_{i\in[k+1]} \alpha_i x_i = \beta\vecu.
\end{equation}
Let $x' = \frac{\alpha_1}{\alpha_1+\alpha_2}x_1 + \frac{\alpha_2}{\alpha_1+\alpha_2}x_2$. Since both $x_1$ and $x_2$ are in $\calG$ by our assumption, $x'$ is also in $\calG$ by the definition of convex set. 
Multiplying both side of the equation above by $\frac{1}{\alpha_1+\alpha_2}$, we obtain

\begin{equation}
\begin{split}
 \quad & \frac{1}{\alpha_1 + \alpha_2}\sum_{i=1}^{k+1}\alpha_ix_i\\
= \quad & \frac{\alpha_1}{\alpha_1+\alpha_2}x_1 + \frac{\alpha_2}{\alpha_1+\alpha_2}x_2 + \frac{1}{\alpha_1+\alpha_2}\sum_{i=3}^{k+1}\alpha_ix_i \\
= \quad & x' + \sum_{i=3}^{k+1}\frac{\alpha_i}{\alpha_1+\alpha_2}x_i \\
= \quad & \frac{\beta}{\alpha_1+\alpha_2}\vecu,
\end{split}
\end{equation}

The equation implies that a set $\tilde{S} = \{x'\}\bigcup S\backslash\{x_1,x_2\}$ is a set of $k$ points in $\calG$ such that a positive linear combination of them gives a vector in the direction of $\vecu$. However, this contradicts our inductive assumption when $n=k$. Therefore, there cannot be a set of $k+1$ points from $\calG$ that satisfies Equation~\ref{eq:1}. The statement we want to prove is also true for $n=k+1$. 
\end{proof}

\subsection{Various Attack Rates in the Intermediate Regime}
\label{sec:intermediatecases}
In Section~\ref{sec:analysis}, we conclude that the poisoning attack's rate can vary from impossible to rapid in the \intermediate\ regime. 
We show three examples in which the attack is very rapid, slow and impossible.

\mypara{The Rapid Case.} \SimplisticAttack\ is shown to be effective against $L_2$-norm defense. 
Let $\tilde{S} = \{(x_0,+1), \cdots, (x_{T-1},+1)\}$ be the poisoning sequence generated by \SimplisticAttack\ against $L_2$-norm defense, which only uses points with $+1$ label. Let $x_{\min}$ denote the point with the smallest $L_2$ norm among $x_0, \cdots, x_{T-1}$. A feasible set $\calF = \{(rx_{x_{\min}},+1)|r\geq 1\}$ contains all points in $\tilde{S}$ but not the origin. An attacker can use $\tilde{S}$ as the poisoning sequence to approach its target as rapidly as against $L_2$-norm defense, even though the feasible set $\calF$ does not contain the origin.

\mypara{The Slow Case.} Notice that the poisoning points in \SimplisticAttack\ normally have diminishing magnitudes towards the end. We now show that if the attacker is only allowed to use poisoning points with constant magnitude, then the attack can be very slow. 

Suppose the feasible set only contains a single point $(x,+1)$ where $x=r\frac{\thetastar-\thetatildezero}{\norm{\thetastar-\thetatildezero}}$. The attacker can only choose $\left(x, +1\right)$ as the poisoning point. The gradient update will be
$$
\theta_{t+1} = \theta_t + \frac{\eta x}{1+\exp(\theta_t^{\top}x)}
$$
for all $t$.

We consider a 1-D example with $\thetatildezero = 0.5$, $\thetastar = 1$, step size $\eta=1$ and an $L_2$-norm upperbound $r=10$. If the attacker can choose any point $(x, +1)$ such that $x \leq r = 10$, the attacker can reach its objective with a single poisoning points by solving the following equation
$$
1 = 0.5 + \frac{x}{1+\exp(0.5x)},
$$
which gives $x \approx 1.629$.

Now suppose the attacker is only allowed to choose $(r, +1)$ as the poisoning point. Let $\Delta_t$ denote the difference between $\theta_{t+1}$ and $\theta_t$. We have 
$$
\Delta_t = \frac{x}{1+\exp(\theta_t^{\top}x)} = \frac{10}{1+\exp(10\theta_t)} \leq \frac{10}{1+\exp(5)}. 
$$
The inequality is because $\theta_t \in [0.5, 1]$ as it starts from $\thetatildezero=0.5$ and approaches $\thetastar = 1$. The attacker then needs at least $\lceil \frac{0.5(1+\exp(5))}{10}\rceil = 8$ poisoning points to achieve its goal. Notice that the exponential term grows much more rapidly then the denominator. If $r=20$, the attacker will need at least $\lceil \frac{0.5(1+\exp(10))}{20}\rceil = 551$ points! In short, the attacker needs exponentially many points w.r.t. $r$ in this case to achieve the same objective that can be done using a single point against $L_2$-norm defense.

This example also naturally extends to high dimensional inputs as long as the attacker only wants to change one coordinate of the model parameter.

\mypara{The Impossible Case.}
We consider the same $\thetatildezero, \thetastar$ and $\eta$ as in the slow case. Now the attacker can only choose $(2.5, +1)$ as the poisoning point, i.e. $r=2.5$. After one update step, the model $\theta$ becomes
$$
\theta = \thetatilde+\frac{\eta x}{1+\exp(\thetatilde^{\top}x)} = 0.5+\frac{2.5}{1+\exp(1.25)} = 1.058.
$$
Notice that the model parameter monotonically increases if the attacker adds more poisoning points of $(2.5, +1)$ into the stream. Therefore, this $\theta$ after one poisoning point is the closest the attacker can ever get to the target $\thetastar = 1$. No data poisoner can output $\thetastar$ exactly. 

\subsection{Proof to Lemma~\ref{lem:l2centroid}}
\label{sec:l2centroidproof}
For simplicity, we consider a feasible set $\calF'^c = \{(yx,+1)|(x,y)\in\calF^c\}$, which is the one-sided form of $\calF^c$. As a consequence of Lemma~\ref{lem:equivalentf}, defenses characterized by $\calF^c$ and $\calF'^c$ have the same behavior. We define $\calF'^c_- = \{(-x,+1)|(x,-1)\in\calF^c_-\}$. Then $\calF'^c = \calF^c_+ \bigcup \calF'^c_-$. Also, let $L(r, \vecu)$ denote a line segment connecting the origin and a point $r\frac{\vecu}{\|\vecu\|}$ as in Theorem~\ref{thm:generalconstraints}. 

\mypara{The Rapid Case.} For the first part, we show that there exists an $r>0$ such that $L(r,\thetastar-\thetatildezero) \subseteq \calF^c_+\bigcup \calF'^c_-$.
Suppose $\norm{\mu_+} < \tau$. Let $\alpha$ denote the angle between $\mu_+$ and $\thetastar-\thetatildezero$, and let $x = l(\thetastar-\thetatildezero)/\norm{\thetastar-\thetatildezero}$ be a point in the direction of $\thetastar-\thetatildezero$ with norm $l$. A point $(x,+1)$ is in $\calF^c_+$ if and only if it satisfies the following inequality
$$
\norm{x-\mu_+}^2 = l^2 - 2l\norm{\mu_+}\cos\alpha + \norm{\mu_+}^2 \leq \tau^2.
$$
We want to find the range for $l$ when the inequality holds. Notice that the determinant of this quadratic inequality, $\norm{\mu_+}^2(\cos\alpha)^2-\norm{\mu_+}^2+\tau^2$, is always positive. Therefore, real solutions always exist for the lower and upper bound of $l$.
Solving the inequality gives us
$$
l \geq \norm{\mu_+}\cos\alpha - \sqrt{\norm{\mu_+}^2(\cos\alpha)^2 - \norm{\mu_+}^2 + \tau^2}
$$
and
$$
l \leq \norm{\mu_+}\cos\alpha + \sqrt{\norm{\mu_+}^2(\cos\alpha)^2 - \norm{\mu_+}^2 + \tau^2}.
$$
 
It is also easy to verify that the lower bound is smaller than $0$ and the upper bound is larger than $0$. Let $r = \norm{\mu_+}\cos\alpha + \sqrt{\norm{\mu_+}^2(\cos\alpha)^2 - \norm{\mu_+}^2 + \tau^2}$. The solution implies that $L(r,\thetastar-\thetatildezero) \subseteq \calF^c_+$.

Similarly, when $\norm{\mu_-} < \tau$, we can show that there exists an $r$ such that $L(r,\thetastar-\thetatildezero)$ is in $\calF'^c_-$. Therefore, when $\tau > \min(\norm{\mu_+}, \norm{\mu_-})$, there exists some $r>0$ s.t. $L(r,\thetastar-\thetatildezero) \subseteq \calF'^c$. The rest follows from Theorem~\ref{thm:generalconstraints}.

\mypara{The Impossible Case.} 
For the second part, we show that $\calG = \{(x,+1)|(\thetastar-\theta_0)^{\top}x \leq 0\}$, i.e. the negative halfspace characterized by $(\thetastar-\theta_0)$, contains both $\calF^c_+$ and $\calF'^c_-$ yet does not intersect $L(+\infty, \thetastar-\theta_0)$ at all. The rest follows from Theorem~\ref{thm:generalconstraints}.

We first look at $\calF^c_+$. The condition $\langle\mu_+, \thetastar-\theta_0\rangle <0$ implies that $\mu_+ \in \calG$. The distance between $\mu_+$ to the hyperplane $(\thetastar-\theta_0)^{\top}x=0$ is $\norm{u_+}$. Since $\norm{u_+} \geq \tau$, none of the point in $\calF^c_+$ can cross the hyperplane, and therefore $\calF^c_+ \subseteq \calG$. The proof for showing $\calF'^c_-$ in $\calG$ is analogous. Therefore $\calG$ contains both $\calF^c_+$ and $\calF'^c_-$. In addition, $\calG$ does not contain any point $x$ in $L(+\infty, \thetastar-\theta_0)$ because $x = r(\thetastar-\theta_0)/\|\thetastar-\theta_0\|$ for some $r>0$, and $(\thetastar-\theta_0)^{\top}x = (\thetastar-\theta_0)^{\top}(r(\thetastar-\theta_0)/\|\thetastar-\theta_0\|) = r > 0$.

\subsection{Proof to Lemma~\ref{lem:slab}}
\label{sec:slabproof}
For simplicity, we consider a feasible set $\calF'^s = \{(yx,+1)|(x,y)\in\calF^s\}$, which is the one-sided form of $\calF^s$. As a consequence of Lemma~\ref{lem:equivalentf}, defenses characterized by $\calF^s$ and $\calF'^s$ have the same behavior. We define $\calF'^s_- = \{(-x,+1)|(x,-1)\in\calF^s_-\}$. Then $\calF'^s = \calF^s_+ \bigcup \calF'^s_-$.

\mypara{The Rapid Case.}
For the first part, we show that $L((\tau-b_+)/\norm{\beta}, \thetastar-\thetatildezero)$ is in $\calF^s_{+}$ or $L((\tau-b_-)/\norm{\beta}, \thetastar-\thetatildezero)$ is in $\calF'^s_{-}$.

Suppose $\tau-b_+>0>-\tau-b_+$. Let $x$ be a point in $L((\tau-b_+)/\norm{\beta}, \thetastar-\thetatildezero)$. Then $x$ can be expressed as $x=r(\thetastar-\thetatilde)/\norm{\thetastar-\thetatilde}$ for some $0<r<(\tau-b_+)/\norm{\beta}$. We know that 
$$
\beta^{\top}x \leq \norm{\beta}r \leq \norm{\beta}(\tau-b_+)/\norm{\beta} = \tau-b_+,
$$
therefore $\beta^{\top}x + b_+ \leq \tau$. On the other hand, $\beta^{\top}x > 0$ by our assumption that $\beta^{\top}(\thetastar-\thetatildezero) > 0$, and $b_+>-\tau$ by the condition $-\tau-b_+<0$. Therefore, $\beta^{\top}x+b_+ > b_+ > -\tau$. We can conclude that the Slab score $|\beta^{\top}x|\leq \tau$, and hence $L((\tau-b_+)/\norm{\beta}, \thetastar-\thetatildezero) \in \calF^s_{+}$.

Similarly, we can show that $L((\tau-b_-)/\norm{\beta}, \thetastar-\thetatildezero)$ is in $\calF'^s_{-}$ when $\tau-b_->0>-\tau-b_-$. 
The rest follows from Theorem~\ref{thm:generalconstraints}, and an example of such a rapid data poisoner is \SimplisticAttack$(\theta_0, S, K, \calF, \thetastar, \epsilon, r)$ in which $\calF = \calF'^s$ and $r = \min((\tau-b_+)/\norm{\beta}, (\tau-b_-)/\norm{\beta})$.

\mypara{The Impossible Case.}
For the second part, we show that $\calG = \{(x,+1)|\beta^{\top}x \leq 0\}$, i.e. the negative halfspace characterized by $\beta$, contains both $\calF^s_+$ and $\calF'^s_-$ yet does not intersect $L(+\infty, \thetastar-\thetatildezero)$ at all. The rest follows from Theorem~\ref{thm:generalconstraints}.

We first look at $\calF^s_+$ when $0 > \tau-b_+ > -\tau-b_+$. 
We know that the Slab score $|\beta^{\top}x+b_+|\leq \tau$ for all $x\in\calF^s_+$. Combining with the condition that $\tau-b_+ < 0$, we have $\beta^{\top}x \leq \tau-b_+ < 0$ for all $x\in\calF^s_+$. Therefore, $\calF^s_+ \subseteq \calG$. The proof for showing $\calF'^s_{-}$ in $\calG$ is analogous. Also, since $\beta^{\top}(\thetastar-\thetatildezero) >0$, we have $\beta^{\top}x > 0$ for all $x\in L(+\infty, \thetastar-\thetatildezero)$. Therefore, $\calG$ does not intersect $L(+\infty, \thetastar-\thetatildezero)$ at all. Theorem~\ref{thm:generalconstraints} tells that no data poisoner can output $\thetastar$ starting from $\thetatildezero$.

If $\beta^{\top}(\thetastar-\theta_0) > 0$, then we can replace $\thetatildezero$ in the above analysis with $\theta_0$, and show that $\calG$ does not intersect $L(+\infty, \thetastar-\theta_0)$ at all. Theorem~\ref{thm:generalconstraints} then tells that no data poisoner can output $\thetastar$.

\subsection{Proof to Lemma~\ref{lem:fullyonlinehard}}
In this problem setting, a linear classifier with classification rule $f(x) = \sgn(\theta^{\top}x)$ is correct on clean examples if $\theta>0$, and is incorrect otherwise.
We show that the online learner's model $\theta_t$ is always positive (except at $t=0$) even in the presence of the attacker, and thus the online model has no classification error on the clean examples for $t>0$.

The update rule of a logistic regression classifier at each iteration $t$ is
$$
\theta_{t+1} = \theta_t + \eta \frac{y_tx_t}{1+\exp(\theta_t^{\top}x_ty_t)}. 
$$ 
For simplicity, we assume $y_t = 1$ for all $t$ because $(x_t,y_t)$ and $(-x_t,-y_t)$ gives the same updates. 
Therefore, the clean example will be $(1, +1)$ and the only feasible poisoning example is $(-1, +1)$.
Substitute $\eta=1$ into the equation, the rule becomes
$$
\theta_{t+1} = \theta_t + \frac{x_t}{1+\exp(\theta_t^{\top}x_t)}. 
$$

Now, let $S$ be any poisoned stream created by the attacker. We use an \emph{interval} to denote a consecutive sequence of data in $S$,
and extract \emph{intervals of interests} using the following rules.

An interval of interest starts at time $t$ if 
1) $\theta_t \geq 0$ and
2) $\theta_{t+1} \leq 0$.
It ends at the first time $t' > t$ when 1) $\theta_{t'} \leq 0$ and 2) $\theta_{t} > 0$.
It also ends if $t'=T$, i.e. we already reach the end the stream.

Notice that by our rule, we can extract multiple intervals of interests from $S$, 
and the intervals will never overlap.

We observe the following two major properties of the intervals.

\mypara{Observation 1.}
Clean examples outside the intervals will all be correctly classified.
\footnote{There exists one corner case: the first point in $S$ is clean and since $\theta_{0}=0$, it has the chance to be misclassified. When this case happen, the error bound on clean examples needs to add $1$, which does not impact the result when the stream length $T$ is large.}

\mypara{Observation 2.}
There are at least as many poisoning points as clean points in each interval.
 
The reason for Observation 1 is straight forward as $\theta_t > 0$ for all $t$ not in the interval.

In order to prove Observation 2, we first show a few useful claims.
\begin{claim}
\label{claim:f1}
Suppose the model $\theta_t \geq 0$ at time $t$, and updates over a poisoning point so that $\theta_{t+1} < 0$.
Then at any $t'$ such that $\theta_{t+1}\leq \theta_{t'}\leq 0$, a single model update over a clean point will always cause $\theta_{t'+1}>0$. 
\end{claim}
\begin{proof}
Deriving from the model update rule, $\theta_{t+1}$ is at least $-0.5$.
This is done by minimizing $\theta_{t} + \frac{-1}{1+\exp(-\theta_t)}$ for $\theta_t \geq 0$.
We again can derive that a single update over a clean point always makes $\theta_{t'+1}>0$
from the model update rule, as $\theta_{t}+\frac{1}{1+\exp(\theta_t)} > 0$ for all $\theta_t\in[-0.5,0]$.
\end{proof}

\begin{claim}
\label{claim:f2}
Suppose $-0.5\leq \theta_t \leq 0$, and in the following $n$ points, there are more clean points than poisoning points.
Then there must be some $t'\in\{t+1,\cdots,t+n\}$ such that $\theta_{t'} > 0$.  
\end{claim}
\begin{proof}
We prove by contradiction. Suppose for all $t'\in\{t+1,\cdots,t+n\}$, $\theta_{t'}\leq 0$. 
Then an update over a poisoning point will lower $\theta$ by at most $0.5$,
while an update over a clean will raise $\theta$ by at least $0.5$.
Since there are more clean points than poisoning points, then the accumulated change in $\theta$
must be at least $0.5$ and strictly more than $0.5$ when $\theta_{t}=-0.5$. Therefore, $\theta_{t+n} > 0$. Contradiction. 
\end{proof}

We then prove Observation 2 by two cases as follows.
 
\mypara{Case 1.} In each interval of interest that ends before time $T$, the first point at the starting time $t$ is always a poisoning point, because $\theta_{t+1}<\theta_t$, i.e. the point moves the model into the wrong direction.
Similarly, the last point at time $t'$ will always be a clean point, because it moves the model in the right direction.
Using the same argument in Claim~\ref{claim:f1}, model $\theta_{t+1}$ is at least $-0.5$. Then by Claim~\ref{claim:f2}, the remaining points cannot contain more clean points than poisoning points,
because otherwise the model $\theta$ will be positive somewhere in the interval, contradicting to the construction of the interval. Therefore, the number of poisoning points at least equal to the number of clean points in the interval.

\mypara{Case 2.} In the interval of interest that ends at time $T$ if such interval exists, the first point at the starting time $t$ is also a poisoning point due to the same argument for the previous case. Once again by Claim~\ref{claim:f2}, the interval must contain no more clean points than poisoning points, otherwise the model $\theta$ will be positive at some time within the interval.

Now combining Observation 1 and 2, we know that the clean points are only classified wrongly in the interval of interests, and the number of such clean points no more than the poisoning points in the intervals. Therefore, no attacker can cause more online error on clean points than the number of poisoning points it can add.

The bound is also tight as the attacker can construct an attack such that the number of error equals to the number of poisoning points. The attack can be done by inserting one poisoning point after every clean point, starting at the beginning of the stream.

\subsection{Proof to Lemma~\ref{lem:fullyonlineeasy}}
We construct a fully-online attacker which adds \emph{one} poisoning example after every ten clean examples, and show that such attacker can achieve the attack effectiveness stated in the lemma.

The main idea of the proof is to show that for most time $t$, $\theta_t$ has very few nonnegative coordinates and thus low classification accuracy on clean examples.
In order to show this, we look at the sum of model parameters over all coordinates, and construct an attacker such that 1) a poisoning example can lower the sum much more rapidly than a clean example increases it when $\theta_t$ has many nonnegative coordinates, and 2) the sum is also lower bounded. If the number of nonnegative coordinates in $\theta_t$ is above a threshold for too many different $t$, then the lower bound will be violated and thus causes contradiction.  

Let $\theta_{t,i}$ denote the $i$-th coordinate of the online model $\theta_t$ at iteration $t$. In this problem setting, the classifier predicts a clean example $(e_i, +1)$ or $(-e_i,-1)$ correctly if $\theta_{t,i} > 0$ and makes an error otherwise. 

The attacker is as follows. At time $t$ when the learner gets a poisoning examples from the attacker, the attacker first picks $1000$ nonnegative coordinates of $\theta_t$. Let $J = \{j_1, \cdots, j_{1000}\}$ denote the indices of these $1000$ coordinates. The poisoning point $x$ has $x_j = -\frac{1}{\sqrt{1000}}$ for $j\in J$ and $x_j = 0$ for other $j\in\mathbb{Z}^+, j\leq 10000$. The label $y = +1$. If fewer than $1000$ coordinates are nonnegative, then the attacker returns a zero vector as the poisoning example, which causes no change to the online model.

We say the online learner makes one cycle of updates when it updates on ten clean examples followed by a poisoning examples, which is the pattern of poisoning locations in the setting. Then 1,000,000 clean examples corresponds to 100,000 complete cycles.

In order to prove the lemma, we first look at the number of nonnegative coordinates in $\theta_t$ when $t=11k+10$ for $k \in \mathbb{Z}^+$, i.e. the time steps when the learner gets an example from the attacker.  
\begin{claim}
\label{claim1}
The number of nonnegative coordinates in $\theta_{11k+10}$ must be below $1000$ in at least 60,000 cycles, i.e. for 60,000 different $k$.  
\end{claim}
\begin{proof}
We prove the contrapositive statement by contradiction.
Let $s_t$ denote the sum of all coordinates of $\theta_t$, i.e. $s_t = \sum_{j=1}^{10000} \theta_{t,j}$, and $\theta_{t,i}$ denote the $i$-th coordinate of the online model $\theta_t$ at iteration $t$. We have the following observations on $\theta_{t,i}$ and $s_t$.

\mypara{Observation 1.} We have $\theta_{t,i} \geq -\frac{1}{\sqrt{1000}}$ for all $t$ and $i$. This is because $\theta_{t,i}$ decreases only when the learner updates over a poisoning example and the coordinate is nonnegative, and the poisoning example can lower $\theta_{t,i}$ by at most $\frac{1}{\sqrt{1000}}$ in one update.

\mypara{Observation 2.} For each update over the nonzero poisoning example, $s_{t+1} - s_t \leq -\sqrt{1000}/2$. This is because for each coordinate $j \in J$, $\theta_{t+1,j} - \theta_{t,j} \leq -1/(2\sqrt{1000})$ by our construction of the attacker, and there are a total of $1000$ coordinates to be updated.
 
Suppose that in more than 40,000 cycles, i.e. 40,000 different $k$, the number of nonnegative coordiates in $\theta_{11k+10}$ is more that $1000$.
Then for at least 40,000 different $t$, $s_{t+1}-s_t \leq -\sqrt{1000}/2$. 
On the other hand, when the learner updates over a clean example, $s_{t+1}-s_t\leq 1/(1+\exp(\theta_t^{\top}x_ty_t)) \leq 0.52$ because $\theta_{t,i}\geq -\frac{1}{\sqrt{1000}}$ for all $i, t$ and $\theta_t^{\top}x_ty_t \geq -\frac{1}{\sqrt{1000}}$.
There are a total of 1,000,000 clean examples.
We can get an upper bound of $s_T$ as follows.
\begin{align*}
s_{T} = s_{T} - s_{0} & =  \sum_{t=0}^{T-1} s_{t+1}-s_t\\
& \leq  40000(-\sqrt{1000}/2) + 1000000(0.52)\\
& \leq -10^5. 
\end{align*}

However, by Observation 1, $\theta_{T,i} \geq -\frac{1}{\sqrt{1000}}$ for all $i$, 
$$
s_T \geq 10000(-\frac{1}{\sqrt{1000}}) \geq -400,
$$
which contradicts the lower bound. 
Hence, it is impossible that $\theta_{11k+10}$ has at least $1000$ nonnegative coordinates for more than 40,000 cycles, and the statement is proven. 
\end{proof}

\begin{claim}
\label{claim2}
At least 600,000 clean examples are classified by models with fewer than 1000 nonnegative coordinates.
\end{claim}
\begin{proof}
Claim~\ref{claim1} shows that $\theta_{11k+10}$ has fewer than 1000 nonnegative coordinates in over 60,000 cycles.
Notice that in each cycle, the number of nonnegative coordinates of $\theta_t$ never decreases when model updates on clean examples.
Therefore, $\theta_{11k+i}$ must have fewer than 1000 nonnegative coordinates too for $i=0,\cdots,9$, i.e. the time steps when the model updates over clean examples.
There are 10 clean examples in each cycle, and hence the statement is proven.
\end{proof}

The last step is to show that with high probability, at least 500,000 out of the 600,000 examples in Claim~\ref{claim2} are misclassified using concentration bound. 
Let $J'$ be the set of time step indices at which the online model has fewer than 1000 nonnegative coordinates. By Claim~\ref{claim2}, $|J'|\geq 600,000$. 
Let $X_t = \mathds{1}(\mbox{$\theta_t$ correctly predicts the clean example $(x_t, y_t)$})$. 
For $t\in J'$, $\mathbb{E}(X_t) < 0.1$ because 1) the input $x_t$ only has one nonzero coordinate, 2) the prediction is only correct when the nonzero coordinate of $x_t$ corresponds to a nonnegative coordinate of $\theta_t$, and 3) $x_t$ is independent of $\theta_t$. In addition, $X_t$ is bounded in $[0, 1]$. 
Let $S = \sum_{t\in J'}X_t$ be the random variable of total number of correct prediction for time steps in $J'$.
Notice that 
$$\mathbb{E}[S] = \sum_{t\in J'}\mathbb{E}[X_t] < 60000$$.
Applying Hoefdding's inequality, we have   
\begin{equation}
\begin{split}
& \quad \Pr[S > 100000]\\
= & \quad \Pr[S-\mathbb{E}[S] > 100000-\mathbb{E}[S]]\\
\leq & \quad \Pr[S-\mathbb{E}[S] > 40000] \\ 
\leq & \quad \exp\left(-\frac{2(40000)^2}{|J'|}\right)\\
\leq & \quad \exp\left(-\frac{2(40000)^2}{600000}\right) < 0.01. 
\end{split}
\end{equation}
The probability that more than 100,000 examples at times steps $t\in J'$ is classified correctly is less than $0.01$. 
Therefore, with probability more than $0.99$, at least 500,000 examples at times steps $t\in J'$ are misclassified, and hence proves the lemma.  

\section{Experiment Details}
\label{sec:experimentdetails}

\subsection{Data Preparation}
\label{sec:datapreparation}

\mypara{UCI Breast Cancer.} %
Initialization, training and test data sets have size $100, 400, 100$ respectively, and the data set is divided in random. 

\mypara{IMDB Review} We train a Doc2Vec feature extractor on the 50000 unlabeled reviews provided in the original data set using the method in~\cite{gensim}. Then we convert reviews in the original training and test set into feature vector of $d=100$ using the extractor. 
Initialization, training and test data sets have size $5000, 10000, 5000$ respectively. The first two data sets are drawn from the original training set in vector representation without replacement. The test set is drawn from the original test set.

\mypara{MNIST 1v7.} We use a subset of the MNIST handwritten digit data set containing images of Digit 1 and 7. 
Initialization, training and test data set have size $1000, 8000, 1000$ respectively. 
The first two data sets are drawn from the original MNIST training sets in random with no replacement. The test set is drawn from the original MNIST test set.

\mypara{fashionMNIST Bag v.s. Sandal.} We use a subset of the fashionMNIST data set containing images of bags and sandal. 
Initialization, training, and test data set have size $1000, 8000, 1000$ respectively. 
The first two data sets are drawn from the original fashionMNIST training sets in random with no replacement. The test set is drawn from the original fashionMNIST test set.

Each data set is normalized by subtracting the mean of all data points and then scaled into $[-1,1]^d$.

\subsection{Detailed Experiment Procedures.}
\label{sec:detailedprocedures}
The learning rate $\eta$ is set to $0.05$ for UCI Breast Cancer and $0.01$ for the other three data sets. The model obtained after making one pass of the training set using OGD algorithm has high accuracy on all data sets.

For Concentrated attack, the attacker needs to divide its attack budget, i.e. number of poisoning points, to points with $+1$ and $-1$ labels. We try three different splits -- all $+1$ points, half-half, all $-1$ points -- and report the best in terms of poisoning effect.

As mentioned in Section~\ref{sec:exp}, Concentrated attack also needs to adapt to the online setting by imposing an order to set of poisoning examples generated for the offline setting. In positive first order, the attacker appends all the points with $+1$ label to the data stream before appending points with $-1$ labels. In negative first order, the attacker appends all the points with $-1$ label to the data stream first instead. In random order, the points are shuffled and appended to the data stream. We try 100 different random orders, and report the best among the positive-first, negative-first and random orders in terms of poisoning effect. 

For poisoning the initialization set used by Slab defense, the attacker also needs to divide its attack budget to points with $+1$ and $-1$ labels. We divide the budget proportional to the fraction of $+1$ and $-1$ points in the clean initialization set.

\subsection{Additional Experiment Results and Plots}
\label{sec:additionalplots}
In this section, we present additional experiment results and plots mentioned in the main body of paper.

\subsubsection{The Semi-Online Setting}
\mypara{Error Rate of the Final Model on Clean Test Examples}
In Section~\ref{sec:exp}, we present the cosine similarity between the final model $\theta$ and the target model $\theta^*$ against defense parameter $\tau$.
Since the target $\theta^*$ has high error rate on test examples, we also plot the test error rate of $\theta$ on clean examples in Figure~\ref{fig:semionlineacc}.
The trend matches the cosine similarity reported in Section~\ref{sec:exp} -- the final model has low error rate when in hard regime, high error rate for large $\tau$ in easy regime, and intermediate error rate for small $\tau$ in easy regime.

\mypara{Cosine Similarity and Test Error Rate for Large Learning Rate.}
In Section~\ref{sec:exp}, we claim that the attacker cannot make $\cos(\theta, \theta^*)$ close to $1$ for small $\tau$ in easy regime because smaller $\tau$ increases the multiplicative factor $C$ in the theory. To validate our claim, we use a larger learning rate $\eta=1$ for all tasks, which reduces the factor $C$. Figure~\ref{fig:semionlinefast} and~\ref{fig:semionlinefastacc} show the cosine similarity to target model and the test error rate against defense parameter $\tau$ respectively. The result matches our expectation as for large $\eta$, a semi-online attacker can make $\theta$ close to $\thetastar$ even for small $\tau$ in easy regime. On the other hand, it still cannot make $\theta$ in the same direction as $\thetastar$ when in hard regime, and the test error rate remains low in hard regime.

\begin{figure*}
\centering
\includegraphics[width=0.9\textwidth]{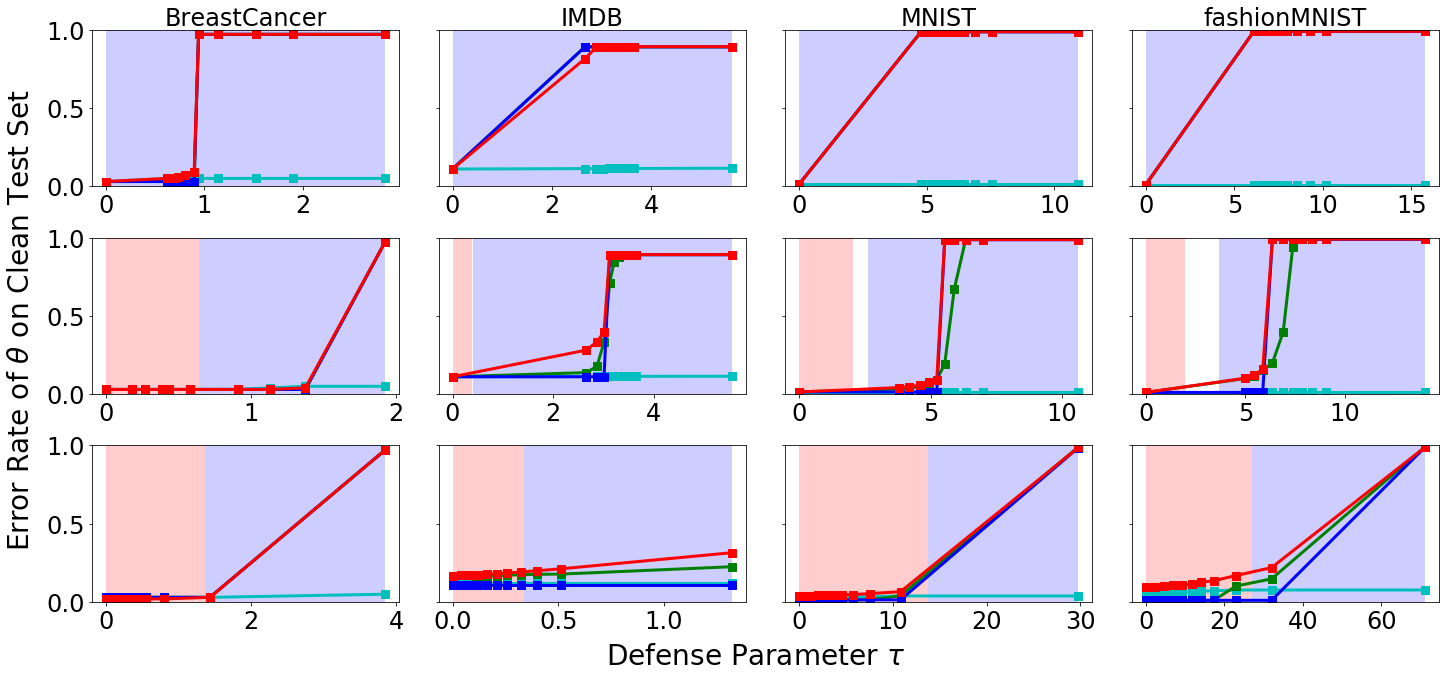}
\includegraphics[height=2.5in]{./figures/semi-legend.png}
\caption{The plot of $\theta$'s error rate on clean test examples against defense parameter $\tau$ for the semi-online attacks.
Small $\tau$ means stronger defense.
Larger error rate means more successful attack.
Red background region indicates hard regime, while blue background indicates easy regime.
{\bf Top to bottom}: $L_2$-norm, $L_2$-distance-to-centroid and Slab defense.
{\bf Left to Right}: BreastCancer, IMDB, MNIST, fashionMNIST.
}
\label{fig:semionlineacc}
\end{figure*}

\begin{figure*}
\centering
\includegraphics[width=0.9\textwidth]{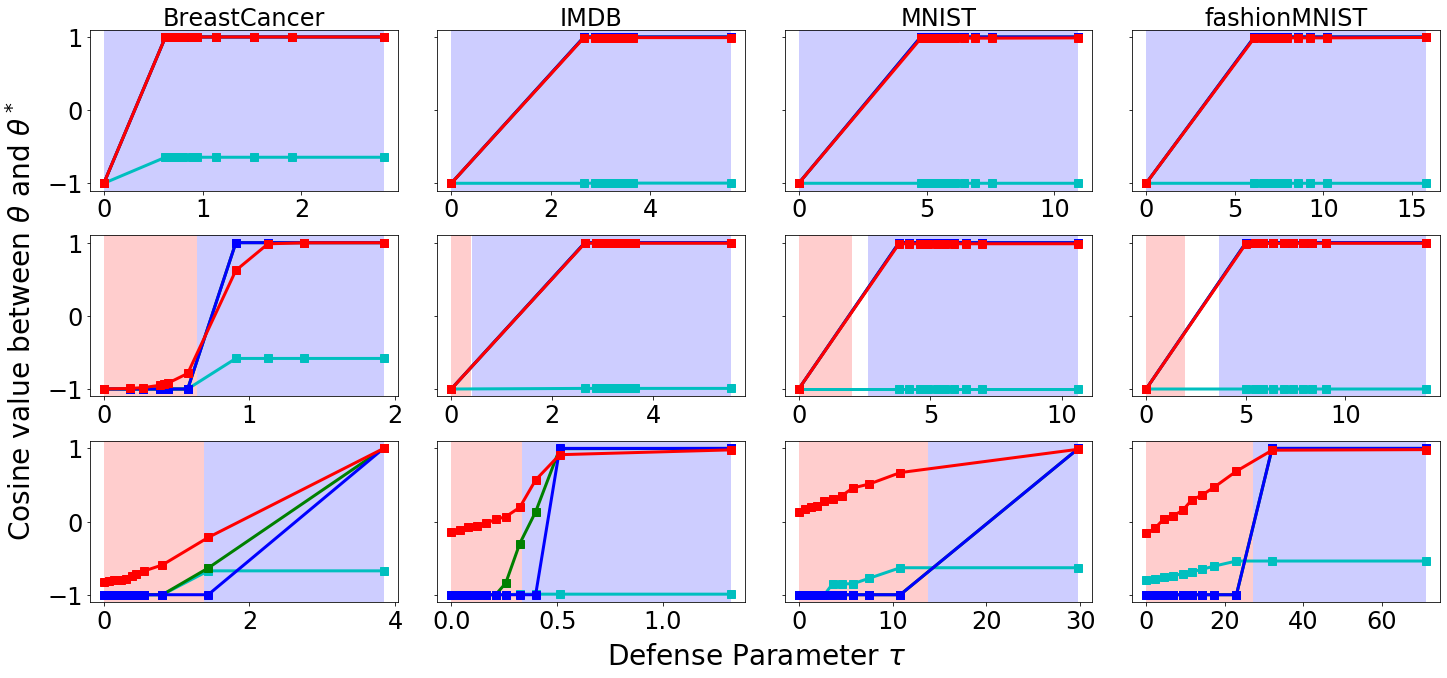}
\includegraphics[height=2.5in]{./figures/semi-legend.png}
\caption{The plot of $\cos(\theta, \theta^*)$ against defense parameter $\tau$ for the semi-online attacks with larger learning rate ($\eta=1$).
Small $\tau$ means stronger defense.
Larger $\cos(\theta, \thetastar)$ means more successful attack.
Red background region indicates hard regime, while blue background indicates easy regime.
{\bf Top to bottom}: $L_2$-norm, $L_2$-distance-to-centroid and Slab defense.
{\bf Left to Right}: BreastCancer, IMDB, MNIST, fashionMNIST.
}
\label{fig:semionlinefast}
\end{figure*}

\begin{figure*}
\centering
\includegraphics[width=0.9\textwidth]{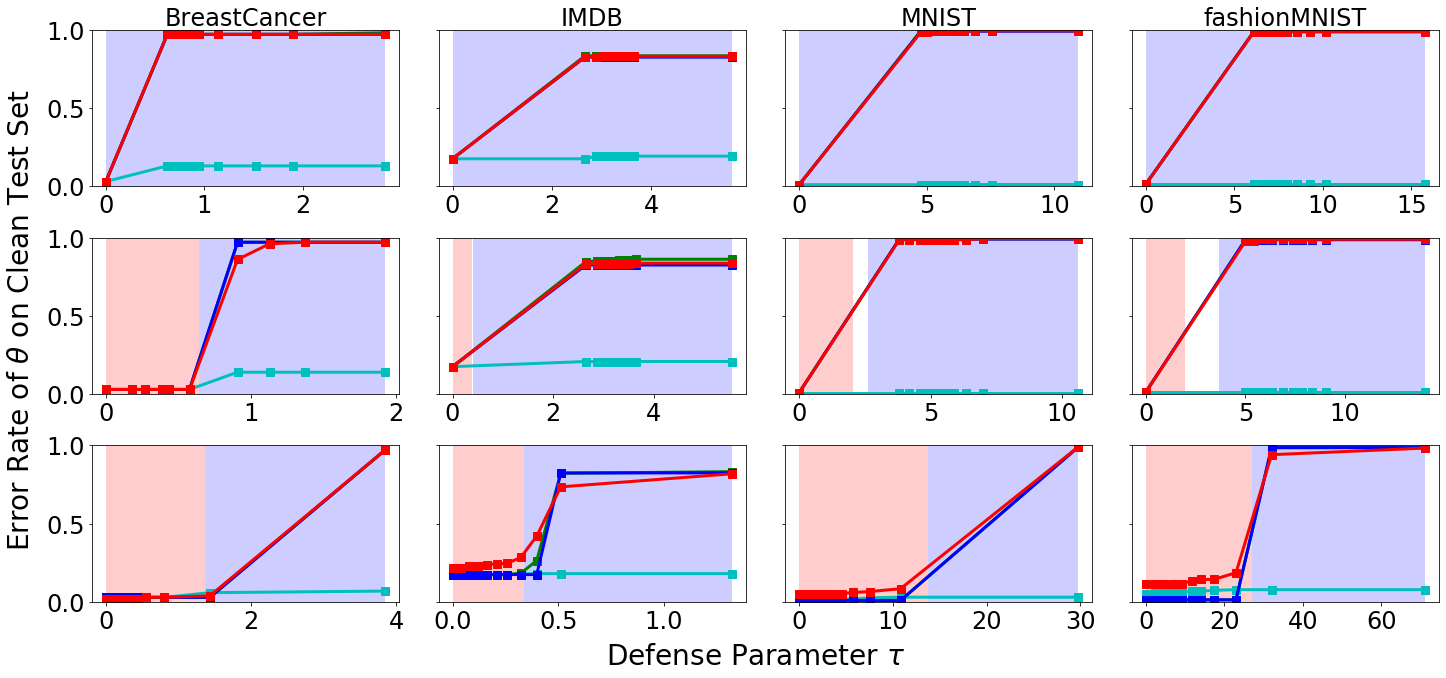}
\includegraphics[height=2.5in]{./figures/semi-legend.png}
\caption{The plot of $\theta$'s error rate on clean test examples against defense parameter $\tau$ for the semi-online attacks with larger learning rate ($\eta=1$).
Small $\tau$ means stronger defense.
Larger error rate means more successful attack.
Red background region indicates hard regime, while blue background indicates easy regime.
{\bf Top to bottom}: $L_2$-norm, $L_2$-distance-to-centroid and Slab defense.
{\bf Left to Right}: BreastCancer, IMDB, MNIST, fashionMNIST.
}
\label{fig:semionlinefastacc}
\end{figure*}

\begin{figure*}
\centering
\includegraphics[width=0.9\textwidth]{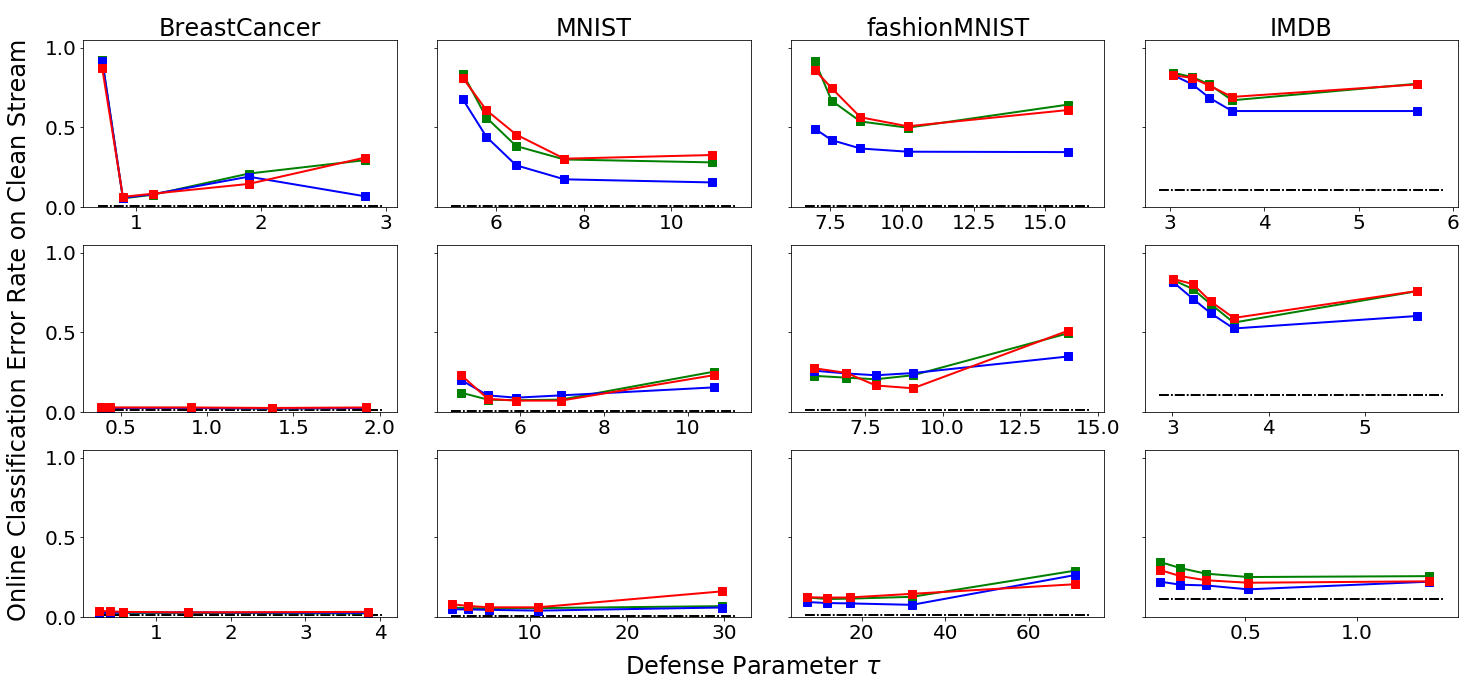}
\includegraphics[height=2.5in]{./figures/fully-legend.png}
\caption{Online classification error rate against defense parameter $\tau$ when 20\% of the examples in the data stream comes from the attackers. 
The dashdot line shows the error rate of the offline optimal classifier.
Small $\tau$ means the defense filters out more examples, both clean and poisoned.
Larger online classification error means more successful attack.
{\bf Top to bottom}: $L_2$-norm, $L_2$-distance-to-centroid and Slab defense.
{\bf Left to Right}: BreastCancer, IMDB, MNIST, fashionMNIST.
The data sets from left to right are in decreasing order of difficulty for successful poisoning attacks.
}
\label{fig:fullyonlineheavy}
\end{figure*}

\subsubsection{The Fully-Online Setting}
In Section~\ref{sec:exp}, we consider a mild fully-online attacker which poisons $10\%$ of the data stream.
We also try a more powerful fully-online attack that can poison $20\%$ of the data stream.
Figure~\ref{fig:fullyonlineheavy} shows the online classification error rate in the presence of this more powerful attacker.
The error rate is high for the weak $L_2$ norm defense. Slab is still able to keep the error rate low and is thus the strongest defense among all three.
The $L_2$-distance-to-centroid defense's error rate is in between, and can typically keep the error rate below the fraction of poisoning examples in the stream.
Meanwhile, the difficulty levels of poisoning attack over the four data sets are in the same order, with BreastCancer be the hardest and IMDB the easiest.

\end{document}